%% file: main_archive.tex
\documentclass[11pt]{article}
\usepackage[left=1in, right=1in, top=1in]{geometry}
\usepackage{xargs}
\usepackage[numbers]{natbib}
\usepackage{fancyhdr}
\usepackage{setspace}
\usepackage{lastpage}
\usepackage{upgreek}
\usepackage{amsmath,mathtools,amsthm,amsfonts,amssymb}	
\usepackage{amssymb,dsfont,bbm}	
\usepackage{xcolor}  
\usepackage{bm}
\usepackage{enumitem}
\usepackage{mathrsfs}

\usepackage{algorithm}
\usepackage{algorithmic}

\setcitestyle{number}

\usepackage[colorlinks=true,breaklinks=true,bookmarks=true,urlcolor=blue,citecolor=blue,linkcolor=blue,bookmarksopen=false,draft=false]{hyperref}

\usepackage{aliascnt}
\usepackage{cleveref}

\usepackage{microtype}
\usepackage{graphicx}
\usepackage{subcaption}
\usepackage{booktabs} 
\usepackage{times}

\usepackage{mathtools}
\usepackage{xargs}
\usepackage{bbm}
\usepackage{enumitem}

\usepackage{hyperref}
\mathtoolsset{showonlyrefs}

\usepackage[textsize=tiny,textwidth=2.0cm]{todonotes}
\newtheorem{assumA}{\textbf{A}\hspace{-1pt}}
\Crefname{assumA}{\textbf{A}\hspace{-1pt}}{\textbf{A}\hspace{-1pt}}
\crefname{assumA}{\textbf{A}}{\textbf{A}}

\newtheorem{assumM}{\textbf{M}\hspace{-1pt}}
\Crefname{assumM}{\textbf{M}\hspace{-1pt}}{\textbf{M}\hspace{-1pt}}
\crefname{assumM}{\textbf{M}}{\textbf{M}}

\newtheorem{theorem}{Theorem}
\crefname{theorem}{theorem}{Theorems}
\Crefname{Theorem}{Theorem}{Theorems}

\newtheorem{lemma}{Lemma}
\crefname{lemma}{lemma}{lemmas}
\Crefname{Lemma}{Lemma}{Lemmas}

\crefname{corollary}{corollary}{corollaries}
\Crefname{Corollary}{Corollary}{Corollaries}

\newaliascnt{proposition}{theorem}
\newtheorem{proposition}[proposition]{Proposition}
\aliascntresetthe{proposition}
\crefname{proposition}{proposition}{propositions}
\Crefname{Proposition}{Proposition}{Propositions}

\newaliascnt{definition}{theorem}

\aliascntresetthe{definition}
\crefname{definition}{definition}{definitions}
\Crefname{Definition}{Definition}{Definitions}

\newaliascnt{definitionProposition}{theorem}

\aliascntresetthe{definitionProposition}
\crefname{Proposition and Definition}{Proposition and Definition}{Proposition and Definition}
\Crefname{Proposition and Definition}{Proposition and Definition}{Proposition and Definition}

\newtheorem{remark}{Remark}
\crefname{remark}{remark}{remarks}
\Crefname{Remark}{Remark}{Remarks}

\crefname{example}{example}{examples}
\Crefname{Example}{Example}{Examples}

\crefname{figure}{figure}{figures}
\Crefname{Figure}{Figure}{Figures}

\input{def.tex}

\usepackage{times}

\begin{document}
\title{Gaussian Approximation for Asynchronous Q-learning}

\author{A. Rubtsov~\footnote{HSE University, Moscow, Russia,  \texttt{asrubtsov@hse.ru}.},  S. Samsonov~\footnote{HSE University, Moscow, Russia,  \texttt{svsamsonov@hse.ru}.},  V. Ulyanov~\footnote{HSE University \& Lomonosov MSU,  Moscow, Russia, \texttt{vulyanov@hse.ru}}, A. Naumov~\footnote{HSE University \& Steklov Mathematical Institute of Russian Academy of Sciences, Moscow, Russia,  \texttt{anaumov@hse.ru}.}}

\maketitle

\begin{abstract}%
In this paper, we derive rates of convergence in the high-dimensional central limit theorem for Polyak–Ruppert averaged iterates generated by the asynchronous Q-learning algorithm with a polynomial stepsize $k^{-\omega},\, \omega \in (1/2, 1]$. Assuming that the sequence of state–action–next-state triples $(s_k, a_k, s_{k+1})_{k \geq 0}$ forms a uniformly geometrically ergodic Markov chain, we establish a rate of order up to $n^{-1/6} \log^{4} (nS A)$ over the class of hyper-rectangles, where $n$ is the number of samples used by the algorithm and $S$ and $A$ denote the numbers of states and actions, respectively. To obtain this result, we prove a high-dimensional central limit theorem for sums of martingale differences, which may be of independent interest. Finally, we present bounds for high-order moments for the algorithm’s last iterate.
\end{abstract}

\section{Introduction}
\input{introduction}

\paragraph{Notations}
We collect here the notation used throughout the paper. For a Markov kernel $\MKQ$ on $(\Xset,\Xsigma)$, and a measurable function $f: \Xset \to \rset$, we set $\MKQ f(x) = \int_{\Xset} f(y) \MKQ(x,\rmd y)$. 
Define also total variation distance $\tvdist(\mu, \nu)$ for probability measures $\mu, \nu$: $
    \tvdist(\mu, \nu) = \tfrac{1}{2}\sup |\mu(f) - \nu(f)|$, where supremum is taken over all functions with $\|f\|_{\infty} \le 1$. For a matrix $\Sigma = \Sigma^{\top} > 0, \Sigma \in \rset^{d \times d}$, we denote by $\lambda_{\max}(\Sigma)$ and
$\lambda_{\min}(\Sigma)$ its largest and smallest eigenvalues, respectively.
We write
$\overline \sigma^2(\Sigma) \coloneqq \max_{j} \Sigma_{jj}$ and
$\underline{\sigma}^2(\Sigma) \coloneqq \min_{j} \Sigma_{jj}$
for the maximal and minimal diagonal entries.
The notation $\|\cdot\|$ and $\|\cdot\|_{\infty}$ refers to the spectral norm and the $\ell_\infty$ norm,
respectively.
We also define the Chebyshev matrix norm by $\|\Sigma\|_{\operatorname{ch}} \coloneqq \max_{i,j} |\Sigma_{ij}|$ and the entrywise \(\ell_1\)
 matrix norm by \(\|\Sigma\|_e = \sum_{i,j}|\Sigma_{ij}|\). We use the notation $\lesssim$
to denote inequalities that hold up to an absolute, problem-independent constant, while \(\lesssim_{\log}\)
additionally suppresses polylogarithmic factors. All vector inequalities are understood componentwise.

\section{Main results}
\label{sec:main_res}

In this section, we give the main results on the non-asymptotic central limit theorem for $Q$-learning iterates. The results are collected in \Cref{section:gar}. At first, we provide some auxiliary statements on non-asymptotic CLT for vector-valued martingales in \Cref{sec:vector-martingales-bounds}.  These results are of independent interest.

\subsection{Gaussian approximation for vector-valued martingales}
\label{sec:vector-martingales-bounds}

This subsection extends the results of \cite{KOJEVNIKOV2022109448} for vector-valued martingales. Let $\{X_k\}_{k\ge 1}$ be a $d$-dimensional martingale difference
sequence with respect to a filtration $\{\mathcal{F}_k\}_{k\ge 0}$, i.e.,
$X_k$ is $\mathcal{F}_k$-adapted, $
\PE[\|X_k\|]<\infty$, where $\|\cdot\|$
is the standard Euclidean norm in $\rset^d$, and
$\PE[X_k \mid \mathcal{F}_{k-1}] = 0$, for any $k \geq 1$. We begin by imposing moment conditions on the martingale increments.
\begin{assumM}
\label{assum:assumptions}
For all $k \ge 1$,  
\(
V_k := \PE\!\left[X_k X_k^\top \mid \mathcal{F}_{k-1}\right]
\)
are well defined and \(
\PE\left[\|X_k\|_\infty^{3}\right] < \infty\).
\end{assumM}
Define
\[S_n= X_1+\ldots+X_n\eqsp, \quad \Sigma_n = \PE[V_1]+ \ldots + \PE[V_n]\eqsp.\]
\begin{assumM}
As a step toward the general CLT, we impose the following assumption on the predictable quadratic variation.
\label{assum:const_variation}
Assume that martingale has almost surely deterministic predictable quadratic variation, that is
\[
\sum_{k=1}^n V_k = \Sigma_n \eqsp \text{ a.s. }
\]
\end{assumM}
We shall measure the quality of approximation in terms of
\begin{equation}
\label{eq:berry-esseen}
d_{K}(X,Y) = \sup_{ h = 1_A, A \in \mathcal{R}} | \PE[h(X)] - \PE[h(Y)]|,
\end{equation}
where $\mathcal R = \{  \prod_{j=1}^d(a_j, b_j], -\infty \le a_j \le b_j \le +\infty\}$ is the class of hyper-rectangulars in $\rset^d$. 
\begin{theorem}
\label{general CLT for martingales with fixed qc}
Under \Cref{assum:assumptions} and \Cref{assum:const_variation}, for any \(d\)-dimensional symmetric positive-definite matrix \(\Sigma\succ 0 \), 
\begin{equation}
d_{K}(S,T)\lesssim
\frac{\log^{5/4}(d)}{\underline{\sigma}^{1/2}(\Sigma_n)}
\Bigg(\overline{\sigma}(\Sigma)+\sum_{k=1}^n\PE\frac{\|X_k\|_\infty^3 }{\lambda_{\min}({P}_k +\Sigma)}\Bigg)^{1/2}\eqsp,
\end{equation}
where \( P_k = \sum_{i=k}^n V_i\eqsp.\)
\end{theorem}
The proof of \Cref{general CLT for martingales with fixed qc} is given in
Appendix~\Cref{section:gar}.
In \cite{KOJEVNIKOV2022109448}, a related result is obtained; however, their analysis is based on the quantity $\lambda = \min_{i \le n}\lambda_{\min}\!\bigl(V_i\bigr)$, which makes  the bound more restrictive. Moreover, the convergence rate $O\big( n^{-1/4}\log^{5/4}(d)\big)$ is derived
under the additional assumption that the conditional covariance matrices $\PE[X_i X_i^\top \mid \mathcal{F}_{i-1}]$ are almost surely deterministic for all $i$. This assumption substantially limits the scope of the result and makes its extension to the general case of random predictable covariances nontrivial. \Cref{lem:clt_general} in the appendix extends \Cref{general CLT for martingales with fixed qc}
to the case of martingales with random predictable variation.

Finally, related results for the Wasserstein distance are established in
\cite{wu2025uncertainty}[Theorem~3.3] and \cite{srikant2025ratesconvergencecentrallimit}[Theorem~1]. In both cases, the resulting convergence rates depend on some 
polynomials in $d$, which might lead to worse bounds in high-dimensional settings.

\subsection{Specification to Q-learning}
We begin this section by specifying the set of assumptions that will be used to derive moment bounds and a non-asymptotic CLT for Q-learning iterates. Consider an infinite-horizon discounted Markov decision process (DMDP) defined by the tuple 
\(\mathcal{M} = (\cS, \cA, \MKQ, r, \gamma)\), where \(\cS\) denotes the state space, \(\cA\) denotes the action space, 
\(\MKQ : \cS \times \cA \to \Delta(\cS)\) represents the transition probability kernel, and \(\gamma \in (0,1)\) indicates the discount factor. 
Throughout, we impose the following regularity condition.

\begin{assumA}[DMDP regularity]
\label{assum:regularity}
Assume that both \(\cS\) and \(\cA\) are finite sets with cardinalities \(S\) and \(A\), respectively. 
Moreover, the immediate reward function \(r(s,a)\) is deterministic and bounded within the interval \([0,1]\).
\end{assumA}
A stationary Markov \emph{policy} is a Markov kernel $\pi:\cS\to\Delta(\cA)$,
\[
s \to \pi(\,\cdot\mid s)=\big(\pi(a\mid s)\big)_{a\in\mathcal{A}},
\quad 
\pi(a\mid s)\ge 0, \quad\sum_{a\in \mathcal{A}}\pi(a\mid s)=1.
\]
Given an initial state $s_0\sim \nu_0$, a policy $\pi$, and the transition kernel $\MKQ$,
the process evolves as $a_t\sim \pi(\cdot\mid s_t)$ and $s_{t+1}\sim \MKQ(\cdot\mid s_t,a_t)$.
The (unregularized) discounted return is
\[
G^\pi = \sum_{t=0}^{\infty}\gamma^t\, r(s_t,a_t)\eqsp,
\qquad 
V^\pi(s) = \E[G^\pi\mid s_0=s]\eqsp.
\]
The performance measure of the agent in classic RL is $V^\pi(s)$. We also define action-value function
$$
Q^\pi(s,a) =  \E[G^\pi\mid s_0=s, a_0 = a]\eqsp .
$$
The goal of the agent is to learn the optimal action-value function 
\(
Q^\star = \max_{\pi} Q^\pi,
\)
in the setting where the transition kernel \(\MKQ\) is unknown. The Q-learning algorithm iteratively updates an approximation of the optimal action–value function 
based on observed data. Assume that the agent only has access to a single  trajectory \(\{s_t,a_t,r_t\}_{t=0}^{T}\) generated under a fixed \emph{behavior policy} \(\pi_b\). 
Formally, the data are generated according to
\[
a_t \sim \pi_b(\cdot \mid s_t), 
\quad r_t = r(s_t,a_t), 
\quad s_{t+1} \sim \MKQ(\cdot \mid s_t,a_t).
\]
We further identify $\MKQ$ with a stochastic matrix of shape $\rset^{SA \times S}$. For any stationary policy \(\pi\), we define the corresponding transition matrix by
\begin{equation}
    \MKQ^{\pi}\bigl((s,a),(s',a')\bigr)
    = \MKQ(s' \mid s,a)\,\pi(a' \mid s') \eqsp.
\end{equation}
The behavior policy induces two important sequences of observations forming Markov chains. The first one is \( z_k = (s_k, a_k) \).  
For this chain, we denote by $\mu$ its stationary distribution, assuming that it exists and is unique. In addition, we will need the Markov chain corresponding to the triplets \( \bar z_k = (s_k, a_k, s_{k+1}) \) on the space  
\(\Xset = \mathcal{S} \times \mathcal{A} \times \mathcal{S} \), where the dynamics of $(s_k,a_k)$ follows $\MKQ^{\pi}$. The associated Markov kernel writes as 
\[
\bar{\MKQ}((s_2,a_2,s_2') | (s_1,a_1,s_1')) = \mathbf{1}\{s_2 = s_1'\} \pi(a_2 | s_2) \MKQ(s_2' \mid s_2, a_2) \eqsp. 
\]
It is easy to verify that the stationary distribution \( \bar{\mu} \) of $\bar{\MKQ}$ chain can be expressed as
\(\bar{\mu}(s,a,s') = \mu(s,a)\MKQ(s'|s, a)\, .
\)
To quantify how actively the behavior policy $\pi_b$ explores the state--action
space, define the minimal visitation probability $\mu_{\min} = \min_{(s,a)} \mu(s,a)$. Sufficient exploration of the transition dynamics, and in particular the
condition $\mumin>0$, is ensured by the following assumption.
\begin{assumA}
\label{assum:UGE}
The Markov kernel \(\bar{\mathrm{\MKQ}}\) admits an absolute spectral gap \(\lambda>0\) and is uniformly 
geometrically ergodic, that is, there exists \(\taumix \in \nset\) such that for all \(t \in \nset\),
\begin{equation}
 d_{\operatorname{tv}}\big( \bar{\MKQ}^t(\cdot | s, a, s'), \bar{\mu} \bigr)
   \;\leq\; \left(1/4\right)^{\lfloor t / \taumix \rfloor}\, .
\end{equation}
\end{assumA}
Asynchronous Q-learning is formalized as follows. The agent initializes \(Q_0\) arbitrarily under the condition 
\(\|Q_0\|\leq  (1-\gamma)^{-1}\). At time \(t\), the agent observes the transition tuple 
\((s_t, a_t, r_t, s_{t+1})\) and updates \(Q_t\) according to
\begin{equation}\label{eq:q_learning_updates}
    Q_{t+1}(s_t,a_t) 
    = Q_t(s_t,a_t) 
      + \alpha_t \Bigl(r_t + \gamma \max_{a\in\cA} Q_t(s_{t+1},a) - Q_t(s_t,a_t)\Bigr),
\end{equation}
while all other entries remain unchanged. The complete algorithm is described in Algorithm \ref{Q learn alg}.

\begin{algorithm}
    \centering
    \caption{Asynchronous Q-learning}\label{Q learn alg}
    \begin{algorithmic}
        \STATE { \textbf{Input parameters:} learning rates $\{\alpha_t\}$, number of iterations $T$.}
        \STATE \textbf{Initialization:} $Q_0$ satisfies $\|Q_0\|_\infty \le (1 - \gamma)^{-1}$.
        \FOR{$t = 0,1,\ldots,T-1$}
          \STATE Draw action $a_{t} \sim \pi_b(s_{t})$
    \STATE Draw next state $s_{t+1} \sim \MKQ(\cdot \mid s_{t}, a_{t})$
    \STATE Update $Q_{t+1}$ according to \eqref{eq:q_learning_updates}.
        \ENDFOR
    \end{algorithmic}
\end{algorithm}

In this paper, we consider the polynomial decaying step sizes $\{\alpha_k\}$.
\begin{assumA}
\label{assum:steps}
Step sizes $\{\alpha_{k}\}_{k \in \nset}$ have a form $\alpha_k = \frac{c_0}{(k+k_0)^{\omega}}$, where $\omega\in (\frac{1}{2}, 1)$ and the initialization parameter $k_0$ satisfies
\begin{equation}
\label{eq:k_0_lower_bound}
k_0^{1-\omega} \asymp \frac{1}{(1-\gamma)\mumin c_0}\eqsp.
\end{equation}
\end{assumA}
We provide technical results on the properties of step-size sequences of the form given in  \Cref{assum:steps} in Appendix, see \Cref{appendix:step_size}.

\subsection{Gaussian approximation for $Q$-learning}
\label{section:gar}
In this section, we analyze the rate of Gaussian approximation for the Polyak--Ruppert averaged iterates of the Q-learning algorithm. Namely, we consider
\begin{equation}
\label{eq:PR_Q_larning}
\bar\Delta_n \coloneqq \frac{1}{n} \sum_{t = 1}^n \{Q_t - Q^\star\}\eqsp.
\end{equation}
We are interested to quantify the rate of convergence w.r.t. the available sample size $n$ and other problem parameters, such as planning horizon \((1-\gamma)^{-1}\), state-action space dimension \(d = SA\) and mixing time \(\taumix\). 
\par 
In order to obtain CLT for $\sqrt{n}\bar\Delta_n$ in \eqref{eq:PR_Q_larning}, we provide a matrix representation of the updates \eqref{eq:q_learning_updates}. We define the operators 
\(\Lam : \cS \times \cA \to \mathbb{R}^{SA \times SA}\) and 
\(\bMKQ : \cS \times \cA\times \cS \to \mathbb{R}^{SA \times S}\) by
\begin{equation}
    \Lam(s,a) = \e_{s,a} (\e_{s,a})^\top, 
    \qquad 
    \bMKQ(s,a,s') = \e_{s,a} (\e_{s'})^\top,
\end{equation}
where \( \e_{s,a} \) and \( \e_s \) denote the canonical basis vectors associated with the state–action pair \((s,a)\) and the state \(s\), respectively.
For brevity, we introduce the random matrices 
\(\Lam_t = \Lam(s_t,a_t)\) and \(\bMKQ_{t} = \bMKQ(s_t,a_t,s_{t+1})\), which are functions of the underlying Markov chain $\bar z_t = (s_t,a_t,s_{t+1})$. With this notation, the update rule takes form
\begin{equation}\label{eq:update_rule}
    Q_{t+1} 
    = Q_t + \alpha_t \bigl( \Lam_t r + \gamma \bMKQ_{t} V_t - \Lam_t Q_t \bigr),
\end{equation}
where \(V_t(s) = \max_{a \in \cA} Q_t(s,a)\). It is straightforward to verify that expectations with respect to the stationary distribution \(\mu\) take the form $\E_{\mu}[\Lam] = D_\mu$, where $D_\mu \in \rset^{SA \times SA}$ is a diagonal matrix of the form $D_\mu = \operatorname{diag}\{\mu(s,a)\}$. Moreover, $\E_{\bar{\mu}}[\bMKQ] = D_{\mu} \MKQ$. To avoid dealing directly with products of random matrices, we employ the following decomposition
\begin{align}
    Q_{t+1} &= Q_t + \alpha_t D_\mu( r+ \gamma \MKQ V_t -  Q_t) + \alpha_t\xi_{t}\eqsp,
\end{align}
where the noise term \(\xi_{t}\) is given by
\begin{align}\label{eq:xi_def}
    \xi_{t} = \gamma (\bMKQ_{t} - D_\mu \MKQ)(V_t - V^\star) - (\Lam_t - D_\mu)(Q_t-Q^\star) +(\Lam_t r+\gamma \bMKQ_{t}V^\star - \Lam_t Q^\star) \eqsp.
\end{align} 
Representation \eqref{eq:xi_def} follows from straightforward algebraic manipulations together with the Bellman equation  
\(Q^\star = r + \gamma \MKQ V^\star\). We also define the Bellman error at optimality 
\begin{equation}
\label{eq:eps_t_def}
\beps(X_t) = \beps_{t} = \Lam_t r+\gamma \bMKQ_{t}V^\star - \Lam_t Q^\star\eqsp.
\end{equation}
Define also \(\Delta_t = Q_t-Q^\star\). In the proof of CLT we will make use of the following assumption, which guarantees that the limit covariance is well defined.
\begin{assumA}
\label{assum:gap}
The discounted Markov decision process \( \mathcal{M} \) admits a unique optimal policy \( \pi^\star \) with a strictly positive optimality gap
\begin{equation}
\kappa_{\mathcal{M}} := 
        \min_{s \in \cS}\,
        \min_{a \ne \pi^\star(s)}
        \bigl|V^\star(s) - Q^\star(s,a)\bigr|
        > 0 \eqsp.
\end{equation}
\end{assumA}
Under assumption \Cref{assum:UGE}, we define the Bellman noise covariance matrix \(\Sigmabf_{\beps}\) as
\begin{align}
\Sigmabf_{\beps} = \E_{\bar{\mu}}[\beps_0\beps_0^\top]
 + 2\sum_{\ell=1}^{\infty}\E_{\bar{\mu}}[\beps_0\beps_\ell^\top]\eqsp.
 \end{align}

In our proof, we proceed with the framework based on the reduction of Markov chains to martingales via the Poisson equation, see \cite[Chapter 21]{douc:moulines:priouret:soulier:2018}. Recall that the Poisson equation associated with the function $\beps(x): \cS \times \cA\times \cS  \to \rset^{SA}$ writes as
\begin{equation}
\label{eq:pois_solution_equation_eps}
\pois{\beps}{}(x) -  \bar{\MKQ} \pois{\beps}{}(x) = \beps(x) \eqsp.
\end{equation}
Under Assumption \Cref{assum:UGE}, equation \eqref{eq:pois_solution_equation_eps} has a unique solution $\pois{\beps}{}: \cS \times \cA\times \cS  \to \rset^{SA}$, which is bounded. We also write $\pois{\beps}{t}$ as an alias for $\pois{\beps}{}(X_t)$. 

\textbf{Gaussian approximation.} Below we present the main result of this section, that is, that \(\sqrt{n}\bar{\Delta}_n\) convergences in distribution to the Gaussian law with the covariance matrix $\Sigmabf_{\infty}$ given by
 \begin{align}
\Sigmabf_{\infty} = G^{-1}\Sigmabf_{\beps}G^{-\top}\eqsp, \quad G = (I-D_\mu(I-\gamma\MKQ^{\pi^\star}))\eqsp.
 \end{align}

\begin{theorem}
\label{thm:clt_q_learning}
Assume \Cref{assum:regularity} - \Cref{assum:gap} and let $Y \sim \mathcal{N}(0,I)$. Then, for
any initialization $Q_0$ satisfying $0 \le Q_0 \le (1-\gamma)^{-1}$  it holds that

\begin{align}
    d_{K}(\sqrt{n}\bar\Delta_n, \Sigmabf_{\infty}^{1/2} Y)&\lesssim_{pr}\frac{\log^4(dn)}{n^{1/2-\omega/2}} + \frac{\log^4(dn)}{n^{\omega - 1/2}} + \frac{\log^{5/4}(d) \log (dn)}{n^{1/4}}
\end{align}
where \(\lesssim_{pr}\) stands for inequality up to absolute and problem-specific constants, but not on $n$ and $d$.
\end{theorem}
\begin{proof}
The proof begins by isolating the leading linear term, as established in \Cref{lem:clt_nonlinear_estimation}
\begin{align}
    \sqrt{n}G\bar\Delta_n = \frac{1}{\sqrt{n}}\sum_{t=1}^n (\pois{\beps}{t} - \totMKQ\pois{\beps}{t-1}) + \mathrm{R}_n^{\operatorname{pr}}\eqsp.
\end{align}
The same lemma yields the bound
\begin{align}\label{eq:clt_pr_rem}
        \PE^{1/p}[\supnorm{\mathrm{R}_n^{\operatorname{pr}}}] \lesssim_{pr}  p^2 \log^2(dn^2)\Big(\frac{1}{n^{\omega/2}} + \frac{1}{n^{1/2-\omega/2}} + \frac{1}{n^{\omega - 1/2}} \Big)\eqsp.
    \end{align}
 Since $\omega \in (1/2,1)$, we have $n^{-\omega/2} \le n^{-(1-\omega)/2}$ for all $n \ge 1$, and hence the first term is dominated by the second and can be omitted.
Define the linear statistic by
\begin{align}
\label{eq: W_n definition}
    W_n  =  \frac{1}{\sqrt{n}}\sum_{t=1}^{n}G^{-1}(\pois{\beps}{t} - \totMKQ\pois{\beps}{t-1})\eqsp.
\end{align}
The associated covariance matrix at time $n$ is given by
\begin{align}
    \Sigmabf_{n}= \E[W_nW_n^\top] = \frac{1}{n}\sum_{t=1}^{n}G^{-1}\E[(\pois{\beps}{t} - \totMKQ\pois{\beps}{t-1})(\pois{\beps}{t} - \totMKQ\pois{\beps}{t-1})^\top]G^{-\top}\eqsp.
\end{align}
 Applying \Cref{lem:shao} with $X = W_n$  and $X^\prime =G^{-1}\mathrm{R}_n^{\operatorname{pr}}$ yields
\begin{align}\label{eq:shao_bound}
d_{K}(\sqrt{n}\bar\Delta_n, \Sigmabf_{\infty}^{1/2} Y) &\leq  2\PE^{1/(p+1)}[\supnorm{G^{-1}\mathrm{R}_n^{\operatorname{pr}}}^p]\left( \frac{2(\sqrt{2 \log d} + 2)]}{\underline{\sigma}(\Sigmabf_\infty)} \right)^{p/(p+1)}\\
&+ d_{K}(W_n, \Sigmabf_\infty^{1/2} Y )\eqsp.
\end{align}
Setting \(p=\log(dn)\) and using estimation \eqref{eq:clt_pr_rem} we bound the first term in \eqref{eq:shao_bound} as follows 
\begin{align}\label{eq:shao_remainder}
&2\PE^{1/(p+1)}[\supnorm{\mathrm{R}_n^{\operatorname{pr}}}^p]\left( \frac{2(\sqrt{2 \log d} + 2)]}{\underline{\sigma}(\Sigmabf_\infty)} \right)^{p/(p+1)} 
\lesssim_{pr}   \frac{\log^4(dn)}{n^{1/2-\omega/2}} + \frac{\log^4(dn)}{n^{\omega - 1/2}} \eqsp.
\end{align}
For the second term in \eqref{eq:shao_bound} we apply triangle inequality
\begin{align}
d_{K}(W_n, \Sigmabf_{\infty}^{1/2} Y ) \leq d_{K}(W_n, \Sigmabf_n^{1/2} Y ) + d_{K}(\Sigmabf_n^{1/2} Y, \Sigmabf_{\infty}^{1/2} Y^\prime)\eqsp,
\end{align}
where $Y^\prime$ is i.i.d. copy of $Y$. Next, \Cref{lem:clt_application} implies that
\begin{align}\label{eq:clt_distance}
    d_{K}(W_n,\,\Sigmabf_n^{1/2}Y) 
    \lesssim_{pr} \frac{\log^{5/4}(d)\log(d^2 n)}{n^{1/4}} \eqsp.
\end{align}
Finally, \Cref{lem:cov_nonstationary_taumix}   guarantees that \[\|\Sigmabf_{n}-\Sigmabf_{\varepsilon}\|_\infty
\lesssim
\frac{\|G^{-1}\|_{\infty}^2\taumix^3}{n(1-\gamma)^2} \eqsp.
\]
Combining this inequality with \Cref{lem:gaussian_comparison} we get
 \begin{align}\label{eq:gaussian_comparison}
     d_K(\mathcal{N}(0, \Sigmabf_{n} ), \mathcal{N}(0,  \Sigmabf_{\beps})) &\lesssim \frac{\lVert \Sigmabf_n - \Sigmabf_{\beps} \rVert_{\infty}}{\lambda_{\min}(\Sigmabf_n)} \log d
\Big( 1 + \Big| \log \frac{\supnorm{\Sigmabf_n - \Sigmabf_{\beps}}}{\lambda_{\min}(\Sigmabf_n)} \Big| \Big)\\
&\lesssim\frac{\log(d)}{n}\frac{\|G^{-1}\|_{\infty}^2\taumix^3}{(1-\gamma)^2\lambda_{\min}(\Sigmabf_n)} \eqsp.
 \end{align}
The statement of the theorem now follows by combining the bounds in
\eqref{eq:shao_remainder}, \eqref{eq:clt_distance} and \eqref{eq:gaussian_comparison}.
\end{proof}

\paragraph{Discussion} Optimising the bound of \Cref{thm:clt_q_learning} over $\omega \in (1/2, 1]$ we get setting $\omega = 2/3$ that
\[d_{K}(\sqrt{n}\bar\Delta_n, \Sigmabf_{\infty}^{1/2} Y)\lesssim_{pr}\frac{\log^4(dn)}{n^{1/6}}\eqsp.\]
\par 
For comparison, \cite{liu2025central} establishes a convergence rate of order
$\widetilde{\mathcal{O}}(\sqrt{d}/n^{1/6})$ in Wasserstein distance.
Therefore the present result extends statistical inference for
\Cref{Q learn alg} to the Kolmogorov distance, without  polynomial
dependence on the dimension of the state-action space.

\subsection{Moment bounds for Q-learning}
\label{sec:moment_bounds}
From Bellman optimality equation
\(Q^\star= r+\gamma \MKQ V^\star\), we get 
\begin{align}\label{eq:Delta_rec_1}
    \Delta_{t+1} = (I - \alpha_t D_\mu)\Delta_t +\alpha_t\gamma D_\mu \MKQ(V_t-V^\star) + \alpha_t\xi_t\eqsp.
\end{align}
To recursively expand  \eqref{eq:Delta_rec_1}, it is convenient to eliminate the term \(\MKQ(V_t - V^\star)\). 
We adopt a standard approach and bound this quantity from above and below in a coordinate-wise manner. 
To this end, we first define the policy \(\pi_t(s) = \arg\max_{a\in\mathcal{A}} Q_t(s,a)\) as the greedy policy induced by the vector \(Q_t\), and denote by \(\pi^\star\) the greedy policy with respect to \(Q^\star\). By construction of the greedy choice, the following inequalities hold
\begin{align}\label{eq:greedy_kernel}
\MKQ^{\pi^\star}Q^\star(s,a) &= Q^\star(s, \pi^\star(s))\geq Q^\star(s, \pi_t(s)) = \MKQ^{\pi_t}Q^\star(s,a)\eqsp, \\
\MKQ^{\pi_t}Q_t(s,a) &= Q_t(s, \pi_t(s)) \geq Q_t(s, \pi^\star(s)) = \MKQ^{\pi^\star}Q_t(s, a)\eqsp.
\end{align}
In vector form, inequalities \eqref{eq:greedy_kernel} can be written as
\[\MKQ^{\pi^\star}Q^\star \geq \MKQ^{\pi_t}Q^\star, \quad \MKQ^{\pi_t}Q_t \geq \MKQ^{\pi^\star}Q_t\eqsp.\]
One can verify that the action–value function  and the corresponding greedy value function
are related through the identities  \(\MKQ V_{t} = \MKQ^{\pi_{t}} Q_{t}\) and \(\MKQ V^\star  = \MKQ^{\pi^\star} Q^\star\).
These relations allow us to connect the term \(\MKQ(V_t - V^\star)\) with \(\Delta_t = Q_t - Q^\star\). 
Specifically,
\begin{align}
    \MKQ(V_{t} - V^\star) &= \MKQ^{\pi_{t}} Q_{t} - \MKQ^{\pi^\star} Q^\star  \leq \MKQ^{\pi_{t}} Q_{t} - \MKQ^{\pi_{t}} Q^\star  = \MKQ^{\pi_{t}} \Delta_{t}\eqsp, \\
    \MKQ(V_{t} - V^\star) &= \MKQ^{\pi_{t}} Q_{t} - \MKQ^{\pi^\star} Q^\star  \geq \MKQ^{\pi^{\star}} Q_{t} - \MKQ^{\pi^{\star}} Q^\star  = \MKQ^{\pi^{\star}} \Delta_{t}\eqsp.
\end{align}
Combining the preceding observations yields the following two-sided bound
\begin{equation}\label{eq:sandwich_1}
    \MKQ^{\pi^\star}\Delta_t 
    \leq \MKQ(V_t - V^\star) 
    \leq \MKQ^{\pi_t}\Delta_t\eqsp.
\end{equation}
Substituting \eqref{eq:sandwich_1} into \eqref{eq:Delta_rec_1}, we get
\begin{align}\label{eq:sandwich_2}
    \Delta_{t+1}
    &\geq (I - \alpha_t D_\mu (I - \gamma \MKQ^{\pi^\star}))\Delta_t
            + \alpha_t \xi_{t}\eqsp, \\
    \Delta_{t+1}
    &\leq (I - \alpha_t D_\mu (I - \gamma \MKQ^{\pi_t}))\Delta_t
            + \alpha_t \xi_{t}\eqsp. 
\end{align}
It is important to highlight the structure of matrices of the form 
\((I - \alpha_t D_\mu (I - \gamma \MKQ^{\pi}))\). 
First, all entries of such matrices are nonnegative, which makes it possible to recursively extend 
inequalities \eqref{eq:sandwich_2} down to the initial deviation \(\Delta_0\). Second, the sum of each row is bounded above by  \(1 - \alpha_t\mu_{\operatorname{min}}(1-\gamma)\), where \(\mu_{\operatorname{min}} = \min_{(s,a)} \mu(s,a)\). 
This, in turn, yields the supremum-norm estimate
\[
\supnorm{I - \alpha_t D_\mu (I - \gamma \MKQ^{\pi})} \leq 1 - \alpha_t\mu_{\operatorname{min}}(1-\gamma)\eqsp.
\]
We now introduce auxiliary sequences \(\{\Delta_t^{(1)}\}_{t\ge 0}\) and \(\{\Delta_t^{(2)}\}_{t\ge 0}\) defined recursively by
\begin{align}\label{eq:Delta_aux_rex}
    &\Delta_{t+1}^{(1)} = (I - \alpha_tD_\mu(I-\gamma\MKQ^{\pi^\star}))\Delta_{t}^{(1)} + \alpha_t\xi_{t}\eqsp, \quad \Delta_0^{(1)} = \Delta_0\eqsp, \\
    &\Delta_{t+1}^{(2)} = (I - \alpha_tD_\mu(I-\gamma\MKQ^{\pi_t}))\Delta_{t}^{(2)} + \alpha_t\xi_{t}\eqsp, \quad \Delta_0^{(2)} = \Delta_0\eqsp.
\end{align}
It is straightforward to verify, using \eqref{eq:sandwich_2} and a simple induction argument, that
\begin{equation}\label{eq:sandwich_3}
\Delta_t^{(1)}\leq\Delta_t\leq\Delta_t^{(2)}\eqsp, \quad \supnorm{\Delta_t}\leq \max\{\supnorm{\Delta_t^{(1)}}, \supnorm{\Delta_t^{(2)}}\}\eqsp.
\end{equation} Accordingly, it suffices to control the norms of the auxiliary sequences. The idea of introducing auxiliary upper and lower bounding sequences for Q-learning iterates was first proposed in \cite{wainwright2019stochastic} and has since become a standard tool for establishing finite-time guarantees.
Despite their similar recursion, \(\Delta_t^{(1)}\) enjoys a key advantage over \(\Delta_t^{(2)}\): 
the previous-step error is propagated through a deterministic matrix. 
In contrast, the randomness of the matrices \(\MKQ^{\pi_t}\) in \(\Delta_t^{(2)}\) 
precludes a direct application of standard martingale concentration tools.

To simplify notation, define the matrix and scalar products
\begin{equation}
\Gamma_{m:n} := \prod_{j=m}^{n}\bigl(I - \alpha_j D_\mu (I - \gamma \MKQ^{\pi^\star})\bigr)\eqsp, \qquad P_{m:n} := \prod_{j=m}^{n}\bigl(1 - \alpha_j \mu_{\min}(1-\gamma)\bigr)\eqsp,
\end{equation}
with the convention that empty products equal \(I\) and \(1\), respectively.
As noted above, by row-sum bounds, we have \(\supnorm{\Gamma_{m:n}} \le P_{m:n}\).
With this notation, the unrolled recursion \eqref{eq:Delta_aux_rex} becomes
\begin{equation}\label{eq:main_xi_dec}
    \Delta_{t+1}^{(1)}
    = \Gamma_{0:t}\,\Delta_0 + \sum_{j=0}^{t} \alpha_j\,\Gamma_{j+1:t}\,\xi_{j}\eqsp.
\end{equation}
The previous observations allow us to establish moment bounds for the convergence of the last iterate of \Cref{Q learn alg}. We refer the reader to \Cref{sec:proof_moments} for the complete derivations.

\begin{theorem}
\label{th:Q_moments_main} 
Assume \Cref{assum:regularity} - \Cref{assum:steps} are met. Then, for any $t>0$ and $p \ge 1$ it holds that
\[\E^{1/p}[\supnorm{Q_t-Q^\star}^p] \lesssim\frac{\log(dt^2)\,p\taumix}{(1-\gamma)^{5/2}t^{\omega/2}\mumin^{3/2}}\eqsp.\]
\end{theorem}
\begin{proof}
We begin by establishing a moment bound for $\Delta_{t+1}^{(1)}$ using the
decomposition \eqref{eq:main_xi_dec}. The noise sequence $\xi_t$ admits the decomposition
$\xi_t=\xi_t^{(0)}+\xi_t^{(1)}$ induced by the Poisson equation; see
\eqref{eq:xi_decompose}. The term $\xi_t^{(0)}$ forms a martingale difference,
whereas $\xi_t^{(1)}$ is a purely Markovian remainder. The martingale term is bounded using the Burkholder–Davis–Gundy inequality (\Cref{lem:bdg_2}), while the Markovian remainder is negligible due to the mixing
properties implied by \Cref{assum:UGE} and is controlled via Minkowski’s
inequality. Having in mind these arguments, we get
\begin{align}\label{eq:xi_bounds_main}\E^{1/p}[\|\sum_{j=0}^{t} \alpha_j\,\Gamma_{j+1:t}\,\xi_{j}^{(0)}\|^p_\infty\lesssim \sqrt{\alpha_t}\,p\log(dt^2)\eqsp, \quad \E^{1/p}[\|\sum_{j=0}^{t} \alpha_j\,\Gamma_{j+1:t}\,\xi_{j}^{(1)}\|^p_\infty\lesssim\alpha_t\eqsp.
\end{align}
For conciseness, all model-dependent constants are suppressed; their explicit
forms can be found in \Cref{lem:xi_0_moments} and \Cref{lem:xi_1_moments}. The transient term $\Gamma_{0:t}\Delta_0$ in \eqref{eq:main_xi_dec} decays
exponentially and, by \Cref{lem:P_alpha_ineq}, satisfies
$\|\Gamma_{0:t}\Delta_0\|_\infty \lesssim \alpha_t$.
This estimate is used in \Cref{lem:Delta_1_union}, which implies that the lower bound in \eqref{eq:sandwich_3} satisfies \begin{align}\label{eq:Delta_0_bound_main}
\PE^{1/p}[\|\Delta_{t+1}^{(1)}\|^p_\infty]\lesssim\sqrt{\alpha_t}\,p\log(dt^2)\eqsp.
\end{align}
Rather than explicitly bounding the upper term in the inequality
\eqref{eq:sandwich_3}, introduce the difference
$\delta_t \coloneqq \Delta_t^{(2)} - \Delta_t^{(1)}$.
The corresponding recursion takes the form
\begin{align}\label{eq:delta_rec_main}
    \delta_{t+1} = \bigl(I - \alpha_t D_\mu (I - \gamma \MKQ^{\pi_t})\bigr)\delta_t
   + \alpha_t \gamma D_\mu \bigl(\MKQ^{\pi_t} - \MKQ^{\pi^\star}\bigr)\Delta_t^{(1)} \eqsp.
\end{align}
The estimate \(\PE^{1/p}[\|\delta_{t+1}\|_\infty^p]\lesssim\sqrt{\alpha_t}\,p\log(dt^2)\) follows from \eqref{eq:Delta_0_bound_main} and Minkowski’s inequality.  Detailed derivation is provided in
\Cref{th:moments}. Consequently, invoking the inequality
\eqref{eq:sandwich_3} yields
\begin{align}
    \E^{1/p}[\supnorm{\Delta_{t+1}}^p]
    &\leq 2\E^{1/p}[\supnorm{\Delta_{t+1}^{(1)}}^p] + \E^{1/p}[\supnorm{\delta_{t+1}}^p]\lesssim\sqrt{\alpha_t}\,p\log(dt^2)\eqsp.
\end{align}
This completes the proof.
\end{proof}
\paragraph{Discussion}
A related result is obtained in \cite{li2024minimax}, where the following bound
is established for linear-rescaled step sizes
\[
\PE\!\left[\|Q_t - Q^\star\|_\infty\right]
\;\lesssim\;
\frac{\log(dt)\,\taumix^{1/2}}{(1-\gamma)^{2}\,t^{1/2}\,\mumin^{1/2}} \eqsp .
\]
A sharper dependence on $\mumin$ and $(1-\gamma)$ is achieved therein through a
more refined analysis of the error recursion $\delta_t$. By contrast, our result completes the statistical inference theory for
Q-learning by covering the remaining class of polynomially decaying step sizes,
which are most relevant for Gaussian approximation. Improving the dependence on the model parameters in \Cref{th:Q_moments_main} is an interesting direction for future research.

\section{Conclusion}
In this work, we establish a non-asymptotic high-dimensional Gaussian approximation for Polyak--Ruppert averaged iterates of asynchronous Q-learning with polynomial step sizes. Under uniform geometric ergodicity of the transition-tuple Markov chain and optimality gap condition, our bounds yield the rate $n^{-1/6}$, polylogarithmic dependence on the state--action dimension and explicit dependence on key problem parameters such as mixing time and the discount factor. A possible direction for future research is improvement of convergence rate to $n^{-1/4} \log^c d$ for some $c > 0$, which seems to be optimal in the Q-learning setting. This will require to use another expansion for \eqref{eq:PR_Q_larning} with $W_n$ being weighted sum of $\pois{\beps}{t} - \totMKQ\pois{\beps}{t-1}$  and bound smallest singular values of corresponding matrices $P_k$ from Theorem \ref{general CLT for martingales with fixed qc}. Another promising direction is to establish the non-asymptotic validity of multiplier bootstrap procedures for approximating the distribution of the rescaled error of the averaged estimator.

\bibliography{references}

\appendix

\newpage
\section{Notations}
\input{notations}

\section{Proof of \Cref{th:Q_moments_main}}\label{sec:proof_moments}
\subsection{Noise Decomposition via Poisson Equation}
\input{noise_decomposition}

\input{moments_estimation}

\section{Gaussian approximation,  comparison and anti-concentration for sums of high-dimensional martingales}\label{appendix:gar}
\input{appendix_martingale_limits}

\section{Auxillary Results}
\input{appendix_auxillary}

\input{appendix_concentration}

\end{document}

%% file: def.tex
\newcommand{\cS}{\mathcal{S}}

\newcommand{\cA}{\mathcal{A}}
\newcommand{\E}{\mathsf{E}}
\newcommand{\Pb}{\mathsf{P}}

\newcommand{\eqsp}{\;}

\newcommand{\e}{\mathbf{e}}
\newcommand{\Lam}{\mathbf{\Lambda}}
\newcommand{\bMKQ}{\mathbf{P}}
\newcommand{\MKQ}{\mathrm{P}}
\newcommand{\mumin}{\mu_{\operatorname{min}}}

\newcommand{\totMKQ}{\bar{\mathrm{P}}}
\newcommand{\beps}{\boldsymbol{\varepsilon}}
\def\nset{\ensuremath{\mathbb{N}}}
\newcommand{\pois}[2]{\boldsymbol{g}^{#1}_{#2}}
\def\taumix{t_{\operatorname{mix}}}
\def\tvdist{\mathsf{d}_{\operatorname{tv}}}

\newcommand{\Bellman}{\mathcal{T}}

\def\nset{\ensuremath{\mathbb{N}}}
\def\rset{\mathbb{R}}
\def\Xset{\mathsf{X}}
\def\Xsigma{\mathcal{X}}
\newcommand{\rmd}{\mathrm{d}}
\newcommand{\Sigmabf}{\boldsymbol{\Sigma}}

\newcommand{\supnorm}[1]{\lVert #1 \rVert_{\infty}}
\def\PE{\mathsf{E}}


%% file: introduction.tex
In this paper, we investigate 
the asynchronous Q-learning algorithm suggested by \cite{watkins1989learning, watkins1992qlearning},
a simple but fundamental 
method in the field of Reinforcement learning developed by \cite{sutton2018reinforcement, bertsekas1996neuro, puterman1994mdp}. 
Q-learning focuses on approximating the optimal action-value function \(Q^\star\) and the associated optimal value function \(V^\star\).  
These functions satisfy a nonlinear system of fixed-point Bellman optimality equation
\begin{align}
\label{eq:bellman_optimality}
r + \gamma \MKQ V^\star = Q^\star\eqsp,
\end{align}
where \(\MKQ\) denotes the transition matrix and \(r\) denotes the reward function. This system characterizes the unique solution corresponding to the optimal control policy. The Bellman optimality equation \eqref{eq:bellman_optimality} allows us to interpret the solution \(Q^\star\) as the fixed point of the Bellman operator  
\(\Bellman\), which is contractive in the discounted setting.  
All variants of the Q-learning algorithm begin with an initial estimate \(Q_0\) and, at each iteration \(t\), update the vector \(Q_t\) according to
\begin{align}
    Q_{t+1} = Q_t + \alpha_t (\widehat{\Bellman} Q_t - Q_t)\eqsp,
\end{align}
where \(\widehat{\Bellman}\) denotes a stochastic approximation of the true Bellman operator.  
The decreasing step-size sequence \((\alpha_t)_{t\in\nset}\) is typically chosen from one of the following three standard families :  
constant, polynomial, or rescaled linear schedules.

Q-learning is the central model-free method in reinforcement learning \cite{wainwright2019variance, li2024q, wang2020randomized, jin2020efficiently}, enabling the computation of optimal value functions without an explicit model of environment dynamics.  
By relying solely on sampled transitions,
model-free algorithms avoid evaluating transition kernels, resulting in a significant reduction in memory usage and computational costs compared to model-based approaches that must solve planning or
control subproblems 
\cite{azar2013minimaxPAC, agarwal20b, NEURIPS2020_96ea64f3}. 

The literature typically distinguishes between two main approaches to data collection.
The first, and conceptually simplest, approach corresponds to the
synchronous setting, in which access to a generative model or simulator allows
one to generate independent samples
\(s' \sim \MKQ(\cdot | s,a)\) for all state-action pairs in each iteration \cite{li2023statistical, wainwright2019stochastic, wainwright2019variance, chen2020finiteSA}. This setup gives rise to what is often called synchronous Q-learning, which is significantly easier to analyze due to the statistical independence of the samples. However, the generative-model assumption is rarely satisfied in applications. The second approach is more realistic: it assumes that the agent has access only to a single trajectory  
\(\{s_t, a_t, r_t\}_{t=0}^T\) generated under a fixed behavior policy \(\pi_b\).  
The corresponding algorithm is known as asynchronous Q-learning \cite{qu2020finite, li2024minimax}.  
In contrast to the synchronous setting, at each iteration only a single coordinate of the vector \(Q_t\) is updated,  
reflecting the fact that the algorithm receives information from only one state–action pair at a time. In what follows, we focus on the asynchronous setting.

Given a sequence of estimates \(\{Q_n\}_{n \in \nset}\), define their Polyak–Ruppert averaged counterparts  \(\{\bar{Q}_n\}_{n \in \nset}\) by
\begin{align}
    \bar{Q}_n = \frac{1}{n}\sum_{t=1}^{n}Q_t\eqsp.
\end{align}
The idea of using averaged estimates $\bar{Q}_n$ was proposed in the works of \cite{ruppert1988efficient} and \cite{polyak1990new, polyak1992acceleration}. The use of averaged iterations $\bar{Q}_n$ 
instead of the last iteration $Q_n$, as has been shown, stabilizes stochastic approximation procedures and
accelerates their convergence.
Moreover, it is known (see \cite{li2023statistical} for the synchronous setting and \cite{liu2025central} for the asynchronous case) that the estimator \(\bar{Q}_n\) is asymptotically normal under suitable regularity conditions on the step-size sequence \((\alpha_n)_{n \in \mathbb{N}}\).  
In particular,
\begin{equation}
\label{eq:CLT_fort_prelim}
\sqrt{n}\,(\bar{Q}_n - Q^\star) \xrightarrow{d} \mathcal{N}(0, \Sigma_{\infty})\eqsp,
\end{equation}
where the covariance matrix \(\Sigma_{\infty}\) is provided in \Cref{section:gar} and captures the asymptotic effect of Bellman noise at the fixed point.

A central line of research focuses on obtaining non-asymptotic properties of Q-learning \cite{li2023statistical, wainwright2019stochastic, wainwright2019variance, li2024q}.  
Much of the existing literature is devoted to establishing moment bounds and concentration inequalities for the  
estimation errors \(Q_n- Q^\star\) and \(\bar{Q}_n - Q^\star\).  
The goal of such results is to derive guarantees with explicit dependence on the number of samples \(n\),  
the state–action dimension \(SA\), and the discount factor \(\gamma\).

Another important direction is the study of the convergence rate in the central limit theorem  
\eqref{eq:CLT_fort_prelim}, quantified via an appropriate metric on probability distributions.  
For instance, the recent work \cite{liu2025central} investigates convergence rates in the Wasserstein distance. In this paper we study the rate of convergence for a class of hyper-rectangulars in $\rset^{S A}$.

The primary motivation for studying the approximation rate in \eqref{eq:CLT_fort_prelim} is the construction of confidence intervals for \(Q^\star\).
A key difficulty is that the asymptotic covariance matrix \(\Sigma_{\infty}\) is unknown in practice, and therefore \eqref{eq:CLT_fort_prelim} cannot be applied directly. 
Classical approaches address this issue by approximating \(\Sigma_{\infty}\) using either plug-in estimators \cite{chen2020aos,wu2024statistical}, or variants of the batch-means method \cite{chen2020aos,zhu2023online_cov_matr,li2024asymptotics}. 
These methods typically construct an estimator \(\hat{\Sigma}_n\) of \(\Sigma_{\infty}\).
An alternative line of recent works \cite{samsonov2024gaussian, sheshukova2025gaussian} considers multiplier bootstrap procedures, adapted from \cite{JMLR:v19:17-370}, which provide non-asymptotic error bounds for coverage probabilities. 
Such approaches avoid relying on the asymptotic distribution of \(\sqrt{n}(\bar{Q}_n - Q^\star)\) and do not require explicit estimation of \(\Sigma_{\infty}\), which is often computationally expensive.
\paragraph{Contributions} Our contributions can be summarized as follows:
\begin{itemize}[noitemsep, nolistsep]
    \item We establish a \emph{non-asymptotic} high-dimensional Gaussian approximation for the
Polyak--Ruppert averaged iterates of asynchronous Q-learning with polynomial step sizes
$\alpha_k \asymp k^{-\omega}$, $\omega\in(1/2,1]$.
Under a uniform geometric ergodicity assumption on the transition-tuple chain
$(s_k,a_k,s_{k+1})_{k\ge 0}$, we bound the Kolmogorov distance over the set of 
hyper-rectangles with leading rate
$n^{-1/6}\,\log^{4}(nSA)$ (up to problem-dependent constants).

\item 
To obtain the Gaussian approximation for Q-learning, we reduce the problem to a general high-dimensional
central limit theorem for sums of vector-valued martingale differences. 
We provide an explicit Gaussian approximation rate for this setting; see Theorem \ref{general CLT for martingales with fixed qc} below. This result is of independent interest and improves the bounds available in the literature (see e.g. \cite{KOJEVNIKOV2022109448}) since it does not depend on local characteristic of individual summands (for example, the lower bounds for  covariances of individual terms). 
\item As an auxiliary result, we derive bounds for high-order moments for the algorithm’s last iterate $Q_t$ under polynomial step sizes. We trace scaling of our bounds with the mixing time of behavioral policy, time horizon $1/(1-\gamma)$ and exploration properties of the behavioral policy. This result is of independent interest as a crucial part in analysis of the remainder term in the CLT and in the bounds for high-order moments.
\end{itemize}
\paragraph{Related works}

In the synchronous setting, the study of moments for error bounds begins with the work   
\cite{evendar2003}, where the authors showed that for polynomial step sizes \(\alpha_k = c_0 / k^{\omega}\), it suffices to take at most  
\(\frac{1}{((1-\gamma)^4\varepsilon)^{1/\omega}}\) iterations to drive the error below \(\varepsilon\).  
A matching upper bound was later obtained in \cite{wainwright2019stochastic} in a worst-case framework using a general  
stochastic approximation approach with cone-contractive operators.  
However, these bounds do not coincide with the minimax lower bound: \cite{azar2013minimaxPAC} demonstrated that 
model-based Q-iteration achieves a sample complexity of order  
\(\frac{1}{(1-\gamma)^3 \varepsilon^2}\), and moreover that this rate is minimax-optimal for any algorithm.

To bridge this gap, \cite{wainwright2019variance} proposed a variance-reduced variant of Q-learning that attains the minimax rate.  
Furthermore, \cite{li2023statistical} showed that Polyak--Ruppert averaged
Q-learning attains the optimal rate under an additional optimality gap
assumption. However, it remained open whether \emph{vanilla} Q-learning itself
can achieve the scaling $\frac{1}{(1-\gamma)^3}$.

This question was resolved recently by \cite{li2024q}, who proved that
Q-learning with a constant step size, or with the linearly rescaled schedule
$\alpha_k = \frac{1}{(1-\gamma)k}$, achieves a sample complexity of order
$\frac{1}{(1-\gamma)^4\varepsilon^2}$. Moreover, they constructed an MDP instance
demonstrating that this bound is unimprovable without optimality gap assumption.
The results of \cite{li2024q} establish convergence rates in both synchronous and asynchronous regimes.

Let us briefly discuss the central limit theorem in $\rset^d$ . Rates of convergence crucially depend on the class of test sets
used to measure distributional distance. For general convex sets and in the i.i.d. setting with $\E[X_1 X_1^\top] = I$, \cite{Bentkus2003} established a dimension-dependent rate of order $ d^{1/4} n^{-1/2} \E \|X_1\|^3$,
which can be loose in high-dimensional regimes.
Motivated by modern high-dimensional inference (e.g., simultaneous confidence regions and
maxima-type functionals), 
many scientists focus 
on hyper-rectangles. In a series of papers, 
\cite{ChernozhukovChetverikovKato2013Gaussian,ChernozhukovChetverikovKato2017aCLTBootstrap, chernozhukov2022improvedcentrallimittheorem} developed high-dimensional  Gaussian approximation theory over hyper-rectangles with
rates that depend only polylogarithmically on $d$. In particular, for $W_n = n^{-1/2} \sum_{i=1}^n X_i$, where $\{X_1,...,X_n\}$
are centered independent random vectors in $\rset^d$ satisfying certain regularity conditions,
\cite{chernozhukov2022improvedcentrallimittheorem} proved the bound of order $n^{-1/4} \log^{5/4}(dn)$. The proofs 
involve smoothing
the maximum function $\max_{1\le j \le d} x_j$ by $F_\beta(x) = \beta^{-1} \log \sum_{j=1}^d e^{\beta x_j}$ for large values of $\beta$.  \cite{10.1214/20-AAP1629} proposed another  method to prove high-dimensional normal
approximations on hyper-rectangles. 
They suggested to combine the approach of \cite{Gotze1991RateConvergenceMultivariateCLT}
in Stein’s method with modifications of an estimate of \cite{AndersonHallTitterington1998Edgeworth} and a smoothing inequality of \cite{BhattacharyaRao1976NormalApproximation}. In particular, this paper suggests to take $\varphi(x, \varepsilon) = \Pb(x + \varepsilon \eta \in A), A \in \mathcal R$, instead of $F_\beta(x)$.  Quantitative CLTs for martingale difference sequences are more delicate than in the i.i.d.\ case,
largely because (i) predictable quadratic variation may be random, (ii) one can't use all the machinery of Stein's method, similar to \cite{10.1214/20-AAP1629}. Instead, one needs to combine Stein's equation with Lindeberg trick. In one dimension, Berry--Esseen
type bounds for martingales in the Kolmogorov distance go back to  \cite{bolthausen1982martingale} (see also \cite{fan2019exact}) , while \cite{rollin2018quantitative}
establishes non-asymptotic bounds in Wasserstein distance. In the multivariate setting, for a class of convex sets the rate of convergence were obtained in \cite{wu2025uncertainty}. \cite{srikant2025ratesconvergencecentrallimit} established bounds in Wasserstein distance. Note that both \cite{wu2025uncertainty} and \cite{srikant2025ratesconvergencecentrallimit} adopt the ideas of \cite{rollin2018quantitative}. \cite{KOJEVNIKOV2022109448} derive Berry--Esseen bounds for vector-valued martingales for
hyper-rectangles, but their result depends on local characteristic of individual summand and can't be directly applied to our setting. Note that result of \cite{KOJEVNIKOV2022109448} is based on smoothing techniques from \cite{10.1214/20-AAP1629} and Lindeberg's trick. To overcome the problem of controlling individual summands we combine approach of \cite{KOJEVNIKOV2022109448} with \cite{rollin2018quantitative}. 

A growing body of work studies quantitative central limit theorems for stochastic
approximation (SA) algorithms. 
For linear SA with i.i.d.\ noise, \cite{samsonov2024gaussian} obtained convergence
rates of order \(n^{-1/4}\) in the Kolmogorov distance, with applications to
temporal-difference learning. These results were subsequently improved in
\cite{butyrin2025improved}, which achieved rates of order \(n^{-1/3}\) for i.i.d.\ linear SA,
and in \cite{wu2024statistical}, which established analogous rates for TD learning.
\cite{samsonov2025statistical} derived Gaussian approximation results for linear SA
with Markovian noise, restricted to one-dimensional projections. In a related direction,
\cite{wu2025uncertainty} established Gaussian approximation results for TD learning
with convergence rate \(n^{-1/4}\) in convex distance.
\cite{kong2025nonasymptotic} established non-asymptotic central limit theorems for
two-time-scale stochastic approximation with martingale noise. More recently,
\cite{butyrin2026gaussian} extended Gaussian approximation results to the case of
Markovian noise and obtained convergence rates of order \(n^{-1/6}\) in convex distance,
with applications to GTD(0) and TDC.
Finally, \cite{liu2025central} established a central limit theorem for Q-learning
with Markovian noise, obtaining convergence rates of order \(n^{-1/6}\) in the
Wasserstein distance.

%% file: notations.tex
\begin{table}[htbp]
\centering
\caption{Notation used throughout the paper}
\label{tab:notation}
\begin{tabular}{ll}
\toprule
\textbf{Notation} & \textbf{Meaning} \\
\midrule
\multicolumn{2}{l}{\textbf{Markov Decision Process}} \\
\midrule
\( \mathcal{M} \) & Markov Decision Process (MDP) \\
\( \cS \) & State space, \( |\cS| = S \) \\
\( \cA \) & Action space, \( |\cA| = A \) \\
\( r(s,a) \) & Reward function \\
\( \MKQ(s'|s,a) \) & Transition kernel of the MDP \( \mathcal{M} \) \\
\(\gamma\) & Discount factor\\
\(  (1-\gamma)^{-1}\) & Horizon\\

\midrule
\multicolumn{2}{l}{\textbf{Policies and Value Functions}} \\
\midrule
\( \pi_b \) & Behaviour policy \\
\( \pi^\star \) & Optimal policy \\
\( Q^\star(s, a) \) & Optimal Q–function \\
\( V^\star(s) \) & Optimal value function \\
\( \mathcal{T} \) & Bellman optimality operator \\

\midrule
\multicolumn{2}{l}{\textbf{Markov Chains and Kernels}} \\
\midrule
\( \mu \) & Stationary distribution of \( \MKQ^{\pi_b} \) \\
\( \mumin \) & \( \min_{(s,a)} \mu(s,a) \) \\
\( \taumix \) & Mixing time of \( \totMKQ \) \\

\midrule
\multicolumn{2}{l}{\textbf{Learning Process}} \\
\midrule
\( \alpha_k \) & Step–size sequence \\
\( Q_t \) & Q–learning iterates \\
\( \bar{Q}_t \) & Polyak–Ruppert averaged iterates \\

\midrule
\multicolumn{2}{l}{\textbf{General Notation}} \\
\midrule
\( \mathbf{1}\{A\} \) & Indicator function of event \(A\) \\
\( \supnorm{\cdot} \) & Supremum norm (vector/matrix) \\
\( \Pb, \PE \) & Probability and expectation \\
\( \operatorname{diag}(X) \) & Diagonal vector of matrix \(X\) \\
\( \e_s \) & One–hot vector corresponding to state \(s\) \\
\(\pois{f}{}\) & Poisson transformation of function \(f\)\\
\bottomrule
\end{tabular}
\end{table}

%% file: noise_decomposition.tex
Recall that the Poisson equation 
\begin{equation}
\label{eq:pois_solution_equation}
\pois{f}{} -  \bar{\MKQ} \pois{f}{} = f - \E_{\bar\mu}[f]\eqsp, 
\end{equation}
under Assumption \Cref{assum:UGE} has a unique solution for any bounded measurable $f$, which is given by the formula 
\[
\pois{f}{} = \sum_{k=0}^{\infty}\left(\bar\MKQ^{k}f - \E_{\bar\mu}[f]\right)\eqsp.
\]
Moreover, using  \Cref{assum:UGE}, we obtain that $\pois{f}{}$ is also bounded with
\begin{align}
\label{eq:markov:norm_bound_tmix}
\supnorm{\pois{f}{}} &= \supnorm{\sum_{k=0}^{\infty}\totMKQ^k f - \E_{\bar\mu}[f]} \leq \sum_{k=0}^{\infty}  \supnorm{\totMKQ^{k}f - \E_{\bar\mu}[f]}\\
&\leq \supnorm{f}\sum_{k=0}^\infty \supnorm{\totMKQ^k - \vec{1}\cdot \bar\mu} \leq \supnorm{f} \sum_{k=0}^{+\infty} (1/4)^{\lfloor k/\taumix \rfloor} \leq (4/3) \taumix \supnorm{f}\eqsp.
\end{align}
We note that the Poisson equation can be considered both for vector-valued and matrix-valued functions, and estimate~\eqref{eq:markov:norm_bound_tmix} holds in both cases.
Throughout this paper, we use a shorthand notation
\begin{equation}
\label{eq:pois_solution_notation}
\pois{f}{k} := \pois{f}{}(\bar z_k) \eqsp.
\end{equation}
The next step is to decompose the noise via the Poisson equation as
\(
    \xi_t = \xi_t^{(0)} + \xi_t^{(1)}\eqsp,
\)
where the first component forms a martingale difference sequence, and the second one is a small remainder
\begin{align}
    \begin{cases}\label{eq:xi_decompose}
        \xi_{t}^{(0)} = (\pois{\varepsilon}{t} - \totMKQ \pois{\varepsilon}{t-1}) + \gamma(\pois{\bMKQ}{t} - \totMKQ \pois{\bMKQ}{t-1})(V_t-V^\star) - (\pois{\Lam}{t} - \totMKQ \pois{\Lam}{t-1})(Q_t-Q^\star)\eqsp,\\[0.75em]
        \xi_{t}^{(1)} = (\totMKQ \pois{\varepsilon}{t-1} - \totMKQ \pois{\varepsilon}{t}) - \gamma(\totMKQ \pois{\bMKQ}{t} - \totMKQ \pois{\bMKQ}{t-1})(V_t- V^\star) + (\totMKQ \pois{\Lam}{t} - \totMKQ \pois{\Lam}{t-1})(Q_t - Q^\star)\eqsp.
    \end{cases}
\end{align}
\begin{remark}\label{rem:preliminary}
    To ensure that the terms \(\xi_0^{(0)}\) and \(\xi_0^{(1)}\) in decomposition \eqref{eq:xi_decompose} are well-defined, we introduce a preliminary term \(\bar z_{-1} = (s_{-1}, a_{-1}, s_0)\), where \(s_{-1} \sim \nu(s_{-1})\) is drawn from the initial state distribution. Note that \(X_{-1}\) is not used in the subsequent algorithmic updates. We also define \(Q_{-1} := Q_0\) and \(V_{-1}:= V_0\).
\end{remark}
Once we split noise, the following decomposition holds
\begin{align}\label{eq:Delta_1_decomposition}
    \Delta_{t+1}^{(1)} = \Gamma_{0:t}\Delta_{0} + \sum_{j=0}^{t}\alpha_j\Gamma_{j+1:t}\xi_j^{(0)}  + \sum_{j=0}^{t}\alpha_j\Gamma_{j+1:t}\xi_j^{(1)}\eqsp.
\end{align}
Next, we derive uniform supremum-norm bounds for the terms appearing in the decomposition \eqref{eq:xi_decompose}.
A classical property of Q-learning is the stability of its iterates. If $0 \le Q_0 \le \frac{1}{1-\gamma}$, then for all $t \ge 0$,
\begin{equation}\label{eq_q_learning_stability}
    \supnorm{V_t-V^\star}\leq \supnorm{Q_t- Q^\star} \leq \frac{1}{1-\gamma}\eqsp.
\end{equation}
Indeed, assume that \eqref{eq_q_learning_stability} holds for iteration \(t\) and verify it for \(t+1\).  
By the update rule \eqref{eq:q_learning_updates},
\begin{align}
    |Q_{t+1}(s_t,a_t) |
    &\leq (1-\alpha_t)|Q_t(s_t,a_t)|
      + \alpha_t |r_t + \gamma \max_{a\in\cA} Q_t(s_{t+1},a)| \\
      &\leq (1-\alpha_t)(1-\gamma)^{-1} + \alpha_t\big(1 + \gamma (1-\gamma)^{-1}\big) = (1-\gamma)^{-1}\eqsp.
\end{align}
A similar argument shows that \(Q_t \ge 0\) for all \(t\).  
Together with the bound \(0 \le Q^\star \le\frac{1}{1-\gamma}\), this yields the desired conclusion.
Moreover, the greedy property of \(V_t\) implies
\begin{align}
    \supnorm{V_t - V^\star} &= \max_{s\in\cS}|V_t(s) - V^\star(s)| =\max_{s\in\cS}|\max_{a\in\cA}Q_t(s,a) - \max_{a\in\cA}Q^\star(s,a)| \\
    &\leq \max_{s, a}|Q_t(s, a) - Q^\star(s,a)| = \supnorm{Q_t - Q^\star}\eqsp.
\end{align}
Applying \eqref{eq:markov:norm_bound_tmix} together with \eqref{eq_q_learning_stability} to each term in the martingale component $\xi_{t+1}^{(0)}$, we obtain the uniform bounds
\begin{align}\label{eq:broot_bouns}
&\|\bMKQ_{t+1}\|_\infty \le 1 \eqsp, 
\qquad 
\|\pois{\bMKQ}{t+1}\|_\infty \le 2\taumix \eqsp, 
\qquad 
\bigl\|(\pois{\bMKQ}{t+1}-\totMKQ\pois{\bMKQ}{t})(V_t-V^\star)\bigr\|_\infty 
\le \frac{4\taumix}{1-\gamma} \eqsp,\\
&\|\Lam_{t+1}\|_\infty \le 1 \eqsp, 
\qquad 
\|\pois{\Lam}{t+1}\|_\infty \le 2\taumix \eqsp, 
\qquad 
\bigl\|(\pois{\Lam}{t+1}-\totMKQ\pois{\Lam}{t})(Q_t-Q^\star)\bigr\|_\infty 
\le \frac{4\taumix}{1-\gamma} \eqsp,\\
&\|\beps_{t+1}\|_\infty \le \frac{2}{1-\gamma} \eqsp, 
\qquad 
\|\pois{\beps}{t+1}\|_\infty \le \frac{4\taumix}{1-\gamma} \eqsp, 
\qquad 
\|\pois{\beps}{t+1}-\totMKQ\pois{\beps}{t}\|_\infty 
\le \frac{8\taumix}{1-\gamma} \eqsp .
\end{align}

The same type of bounds holds for the Markovian noise component \( \xi_{t+1}^{(1)} \)
\begin{align}\label{eq:broot_bouns_2}
   &\bigl\|\totMKQ \pois{\varepsilon}{t} - \totMKQ \pois{\varepsilon}{t+1}\bigr\|_\infty
\le \frac{8\,\taumix}{1-\gamma} \eqsp,\\[0.3em]
&\bigl\|(\totMKQ \pois{\bMKQ}{t+1} - \totMKQ \pois{\bMKQ}{t})(V_t- V^\star)\bigr\|_\infty
\le \frac{4\,\taumix}{1-\gamma} \eqsp,\\[0.3em]
&\bigl\|(\totMKQ \pois{\Lam}{t+1} - \totMKQ \pois{\Lam}{t})(Q_t - Q^\star)\bigr\|_\infty
\le \frac{4\,\taumix}{1-\gamma} \eqsp .
\end{align}
The following lemma establishes moment bounds for the martingale component.
\begin{lemma}
\label{lem:xi_0_moments}
Assume \Cref{assum:regularity}-\Cref{assum:steps}. Then it holds that
\begin{align}
\E^{1/p}[\supnorm{\sum_{j=0}^t\alpha_j\Gamma_{j+1:t}\xi_j^{(0)}}^p] &\lesssim \frac{\alpha_t^{1/2}p\taumix\log(2dt^2)}{\sqrt{(1-\gamma)^{3}{\mumin}}}  \eqsp.
\end{align}
\end{lemma}
\begin{proof}
    Denote by \( \mathcal{H}_j = \sigma(\bar z_0, \ldots, \bar z_j) \)  
the sigma-algebra generated by all observations up to time \( j \). The approximations \(Q_j\) and \(V_j\) are \(\mathcal{H}_{j-1}\)-measurable. Moreover, by construction, the sequence \(\xi_j^{(0)}\) forms a martingale difference sequence with respect to the filtration \(\{\mathcal{H}_j\}_{j=0}^{t-1}\): \(\E[\xi_j^{(0)}\mid \mathcal{H}_{j-1}] = 0\) . By definition \(\xi_t^{(0)}\) and using \eqref{eq:broot_bouns} we conclude that
\begin{align}
     \supnorm{\xi_{t}^{(0)}} &\leq \supnorm{\pois{\varepsilon}{t} - \bar\MKQ \pois{\varepsilon}{t-1}} + \gamma\supnorm{(\pois{\bMKQ}{t} - \bar\MKQ \pois{\bMKQ}{t-1})(V_t-V^\star)} \\&
     + \supnorm{(\pois{\Lam}{t} - \bar\MKQ \pois{\Lam}{t-1})(Q_t-Q^\star) } 
     \leq \frac{16\taumix}{1-\gamma}\eqsp.
\end{align}
Combining this bound with \Cref{lem:P_alpha_ineq}, we obtain the following uniform bound \[\supnorm{\alpha_j\Gamma_{j+1:t}\xi_j^{(0)}} \leq \frac{16\alpha_t\taumix}{1-\gamma}\eqsp.\]
Then, by Lemma \ref{lem:bdg_2}
\begin{align}\label{eq:mds_lem_2}
\E^{1/p} [\supnorm{\sum_{j=0}^t\alpha_j\Gamma_{j+1:t}\xi_j^{(0)}}^p] &\lesssim p\log(2dt^2)\,\Big(\frac{\alpha_t\taumix}{1-\gamma} + \Big(\sum_{j=0}^{t} \E^{2/p}[\supnorm{\alpha_j\Gamma_{j+1:t}\xi_j^{(0)}}] \Big)^{1/2}\Big)\\
 &\lesssim p\taumix\log(2dt^2)\,\Big(\sum_{j=0}^{t}\alpha_j^2P_{j+1:t}^2 (1-\gamma)^{-2} \Big)^{1/2}\\
&\lesssim \frac{\alpha_t^{1/2}p\taumix\log(2dt^2)}{\sqrt{(1-\gamma)^{3}{\mumin}}} \eqsp,
    \end{align}
    where the last inequality follows from \Cref{lem:rate_of_convergence}.
\end{proof}
The lemma below establishes uniform moment bounds for the Markovian component.
\begin{lemma}
\label{lem:xi_1_moments}
Assume \Cref{assum:regularity}-\Cref{assum:steps}. Then it holds that
\begin{equation}
\E^{1/p} [\supnorm{\sum_{j=0}^t\alpha_j\Gamma_{j+1:t}\xi_j^{(1)}}^p] \lesssim \frac{\alpha_t \taumix}{(1-\gamma)^{2}\mumin}\eqsp.
\end{equation}
\end{lemma}
\begin{proof}
To facilitate the application of \Cref{lem:telescope}, we introduce the following sequence for \(t \geq -1\)
\begin{equation}\label{eq:v_definition}
    v_t = -\totMKQ\pois{\beps}{t} - \gamma\totMKQ\pois{\bMKQ}{t}(V_t-V^\star) + \totMKQ\pois{\Lam}{t}(Q_t-Q^\star)\eqsp.
\end{equation}
In accordance with \Cref{rem:preliminary}, the initial conditions are defined as \(V_{-1} = V_0\) and \(Q_{-1} = Q_0\), ensuring that \(v_{-1}\) is well-defined. Recall definition of the remainder term
\begin{equation}
    \xi_{t}^{(1)} = (\bar\MKQ \pois{\varepsilon}{t-1} - \totMKQ \pois{\varepsilon}{t}) - \gamma(\bar\MKQ \pois{\bMKQ}{t} - \totMKQ \pois{\bMKQ}{t-1})(V_t- V^\star) + (\bar\MKQ \pois{\Lam}{t} - \totMKQ \pois{\Lam}{t-1})(Q_t - Q^\star)\eqsp.
\end{equation}
Decompose \(\xi_{j}^{(1)}\) for \(j\geq 0\) as follows
\begin{align}
    \xi_j^{(1)} &= (v_j - v_{j-1}) - \gamma\totMKQ\pois{\bMKQ}{j-1}(V_{j-1} - V^\star) + \gamma\totMKQ\pois{\bMKQ}{j-1}(V_{j} - V^\star) \\
    &-\totMKQ\pois{\Lam}{j-1}(Q_j - Q^\star) + \totMKQ\pois{\Lam}{j-1}(Q_{j-1} - Q^\star) \\
    &= (v_j - v_{j-1}) + \gamma \totMKQ\pois{\bMKQ}{j-1}(V_{j} - V_{j-1}) - \totMKQ\pois{\Lam}{j-1}(Q_j - Q_{j-1})\eqsp.
\end{align}
From \eqref{eq_q_learning_stability} and  \eqref{eq:broot_bouns_2} follows that \(\E^{1/p}[\supnorm{v_j}^p] \lesssim \frac{\taumix}{1-\gamma}  \) and \( \E^{1/p}[\supnorm{\xi_j}^p] \lesssim\frac{1}{1-\gamma}  \). From the update rule specified in \eqref{eq:update_rule} and \eqref{eq_q_learning_stability}, we derive the following bound
\begin{align}
    \E^{1/p}[\supnorm{V_{j+1} - V_{j}}^p] &\lesssim\E^{1/p}[\supnorm{Q_{j+1} - Q_{j}}^p]\\
    &\lesssim \alpha_{j} \big(\E^{1/p}[\supnorm{\xi_{j}}^p] + \E^{1/p}[\supnorm{r + \gamma \MKQ V_t}^p \big) \\
    &\lesssim \alpha_{j} \big(\E^{1/p}[\supnorm{\xi_{j}}^p] + \E^{1/p}[\supnorm{(Q_j - Q^\star) + \gamma\MKQ(V_t-V^\star)}^p \big) \\
    &\lesssim \frac{\alpha_j}{(1-\gamma)} \eqsp,
\end{align}
Next we substitute \eqref{eq:v_definition} in summation and use \Cref{lem:telescope}:
\begin{align}\label{eq:decomposition_of_remain}
\sum_{j=0}^{t}\alpha_j\Gamma_{j+1:t}\xi_{j}^{(1)} &= \sum_{j=0}^{t}\alpha_j\Gamma_{j+1:t}\left((v_j - v_{j-1}) + \gamma \totMKQ\pois{\bMKQ}{j-1}(V_{j} - V_{j-1}) - \totMKQ\pois{\Lam}{j-1}(Q_j - Q_{j-1})\right) \\
&= \alpha_tv_t - \alpha_0P_{1:t}v_{-1} + \sum_{j=0}^{t-1}\big((\alpha_j - \alpha_{j+1}) - \alpha_j\alpha_{j+1}D_{\mu}(I - \gamma\MKQ^{\pi^\star})\big)\Gamma_{j+2:t}v_j\\
&+ \sum_{j=0}^{t}\alpha_j\Gamma_{j+1:t}\left( \gamma \totMKQ\pois{\bMKQ}{j-1}(V_{j} - V_{j-1}) - \totMKQ\pois{\Lam}{j-1}(Q_j - Q_{j-1})\right) \eqsp.
\end{align}
Applying Minkowski’s inequality to \eqref{eq:decomposition_of_remain} and using the fact that  
\(\alpha_j - \alpha_{j+1} \leq \alpha_j^2\) (cf. \Cref{lem:alpha_delta}), we obtain
    \begin{align}
        \E^{1/p}[\|\sum_{j=0}^{t}\alpha_j\Gamma_{j+1:t}\xi_j^{(1)}\|^p] &\lesssim \frac{\alpha_t\taumix}{1-\gamma}  + \frac{\taumix}{1-\gamma} \sum_{j=1}^t \alpha_j^2P_{j+1:t} \lesssim \frac{\alpha_t\taumix}{(1-\gamma)^2\mumin}\eqsp,
    \end{align}
where in the last step we use \Cref{lem:rate_of_convergence}
\end{proof}

Now we combine results of \Cref{lem:xi_0_moments} and \Cref{lem:xi_1_moments}:
\begin{lemma}
\label{lem:Delta_1_union}
Assume \Cref{assum:regularity}-\Cref{assum:steps}. Then it holds that
\begin{align}
\E^{1/p}[\|\Delta_{t+1}^{(1)}\|^p] \lesssim \frac{\alpha_t^{1/2}p\taumix\log(2dt^2)}{\sqrt{(1-\gamma)^{3}\mumin} }\eqsp,
\end{align}
\end{lemma}
\begin{proof}
We apply Mikowsi inequality to decomposition \eqref{eq:Delta_1_decomposition} and use 
 \Cref{lem:xi_0_moments} and \Cref{lem:xi_1_moments}:
 \begin{align}
 \E^{1/p}[\supnorm{\Delta_{t+1}^{(1)}}^p] &\leq
     \E^{1/p}[\supnorm{\Gamma_{0:t}\,\Delta_0}^p] + \E^{1/p}[\supnorm{\sum_{j=0}^{t} \alpha_j\,\Gamma_{j+1:t}\,\xi_{j}^{(0)}}^p] + \E^{1/p}[\supnorm{\sum_{j=0}^{t} \alpha_j\,\Gamma_{j+1:t}\,\xi_{j}^{(1)}}^p]\\
     &\lesssim\E^{1/p}[\supnorm{\Gamma_{0:t}\,\Delta_0}^p]  + \frac{\alpha_t^{1/2}p\taumix\log(2dt^2)}{\sqrt{(1-\gamma)^{3}{\mumin}}}  + \frac{\alpha_t \taumix}{(1-\gamma)^{2}\mumin}\eqsp.
\end{align} Finally note that by \Cref{lem:P_alpha_ineq}
\[\E^{1/p}[\supnorm{\Gamma_{0:t}\,\Delta_0}^p]  \lesssim \frac{P_{0:t}}{1-\gamma} \lesssim \frac{\alpha_t}{\alpha_0(1-\gamma)}\lesssim\frac{\alpha_t}{(1-\gamma)^2\mumin}\eqsp.\]
The claim follows from the monotonicity of the stepsizes $\alpha_t \le \alpha_0$
together with \Cref{assum:steps} which implies \(k_0^{-\omega}\leq k_{0}^{\omega-1}\lesssim(1-\gamma)\mumin\).
\end{proof}

%% file: moments_estimation.tex
Above we bound moments of the sandwich lower bound \(\E^{1/p}[\|\Delta_j^{(1)}\|^p]\). The following lemma derived a bound on the true error moments using  \(\E^{1/p}[\|\Delta_j^{(1)}\|^p]\).
\begin{lemma}
\label{th:moments}
Assume \Cref{assum:regularity}-\Cref{assum:steps}. Then it holds that
\begin{equation}
\E^{1/p}[ \|\Delta_{t+1}\|^p] \lesssim \frac{\alpha_t^{1/2}p\taumix\log(2dt^2)}{\sqrt{(1-\gamma)^{5}\mumin^3} }\eqsp.
\end{equation}
\end{lemma}
\begin{proof}
    We begin by applying the Minkowski inequality to the sandwich bound \eqref{eq:sandwich_3}:
    \begin{align}
    \E^{1/p}[\supnorm{\Delta_{t+1}}^p]&\leq \E^{1/p}[\max\{\supnorm{\Delta_{t+1}^{(1)}}, \supnorm{\Delta_{t+1}^{(2)}}\}^p]\\ 
    &\leq \E^{1/p}[(\supnorm{\Delta_{t+1}^{(1)}} + \supnorm{\Delta_{t+1}^{(2)}})^p] \\
    &\leq  \E^{1/p}[\supnorm{\Delta_{t+1}^{(1)}}^p] + \E^{1/p}[\supnorm{\Delta_{t+1}^{(2)}}^p]\\
    &\leq 2\E^{1/p}[\supnorm{\Delta_{t+1}^{(1)}}^p] + \E^{1/p}[\supnorm{\delta_{t+1}}^p]\eqsp.
\end{align}
Next, we provide moment estimation for 
\(\delta_t := \Delta_t^{(2)} - \Delta_t^{(1)}\),
whose recursion admits a simple structure:
\begin{align}\label{eq:delta_rec}
    \delta_{t+1} = \bigl(I - \alpha_t D_\mu (I - \gamma \MKQ^{\pi_t})\bigr)\delta_t
   + \alpha_t \gamma D_\mu \bigl(\MKQ^{\pi_t} - \MKQ^{\pi^\star}\bigr)\Delta_t^{(1)} \eqsp.
\end{align}
Unrolling recursion \eqref{eq:delta_rec} up to \(t = 0\) and using \(\delta_0 = 0\), we obtain
\begin{equation}
    \E^{1/p}\big[\|\delta_{t+1}\|^p\big]
    \le 2\sum_{j=0}^{t} \alpha_j P_{j+1:t}\, \E^{1/p}\big[\|\Delta_j^{(1)}\|^p\big]\eqsp.
\end{equation}
Substituting bound on \(\E^{1/p}[\supnorm{\Delta_j^{(1)}}^p]\) from \Cref{lem:Delta_1_union} we get:
\begin{align}
    \E^{1/p}\big[\|\delta_{t+1}\|^p\big] \lesssim  \frac{p\taumix\log(2dt^2)}{\sqrt{(1-\gamma)^3\mumin}}\sum_{j=0}^{t} \alpha_j^{3/2} P_{j+1:t} \lesssim \frac{p\taumix\log(2dt^2)}{\sqrt{(1-\gamma)^5\mumin^3}}\eqsp,
\end{align} 
where last inequality uses \Cref{lem:rate_of_convergence}. The desired bound now follows from combination of bounds on \( \E^{1/p}\big[\|\delta_{t+1}\|^p\big]\) and \( \E^{1/p}\big[\|\Delta_{t+1}^{(1)}\|^p\big]\)
\end{proof}

The following lemma provides moment bounds for the nonlinear terms arising in the
Polyak–Ruppert decomposition.
\begin{lemma}
\label{lem:clt_nonlinear_estimation}
Assume \Cref{assum:regularity}-\Cref{assum:gap}. Then the Polyak--Ruppert error admits the decomposition
\begin{align}
\sqrt{n}G\bar\Delta_n = \frac{1}{\sqrt{n}}\sum_{t=1}^n (\pois{\beps}{t} - \totMKQ\pois{\beps}{t-1}) + \mathrm{R}_n^{\operatorname{pr}}\eqsp,
\end{align}
where the remainder term satisfies the bound
    \begin{align}
        \PE^{1/p}[\supnorm{\mathrm{R}_n^{\operatorname{pr}}}] \lesssim  p^2 \log^2(2dn^2)\Big(\frac{\operatorname{B}_1}{n^{\omega/2}} + \frac{\operatorname{B}_2}{n^{1/2-\omega/2}} + \frac{\operatorname{B}_3}{n^{\omega - 1/2}} \Big)
    \end{align}
and the constants \(\operatorname{B}_1,\operatorname{B}_2\) and \(\operatorname{B}_3\) are given by
\begin{equation}
    \operatorname{B}_1= \frac{\taumix^2}{(1-\omega)^{1/2}(1-\gamma)^{3}\mumin^{2} } \eqsp,
    \quad
    \operatorname{B}_2 = \frac{ \taumix}{(1-\gamma)^{7/2}\mumin^{5/2}}  \eqsp,\quad
    \operatorname{B}_3 =\frac{\taumix^2}{(1-\omega)(1-\gamma)^6\mumin^4 \kappa_{\mathcal{M}}} \eqsp.
\end{equation}
\end{lemma}
\begin{proof}
Recall the decomposition~\eqref{eq:Delta_rec_1}:
\begin{align}
     \Delta_{t+1} = (I - \alpha_t D_\mu)\Delta_t +\alpha_t\gamma D_\mu \MKQ(V_t-V^\star) + \alpha_t\xi_t\eqsp.
\end{align}
Next, the Bellman operator is linearized around the optimal point, yielding
\begin{align}\label{eq:bellman_nonlinear}
    \MKQ(V_t - V^\star) = \MKQ^{\pi_t}Q_t - \MKQ^{\pi^\star} Q^\star = \MKQ^{\pi^\star}\Delta_t + (\MKQ^{\pi_t} - \MKQ^{\pi^\star})Q_t\eqsp.
\end{align}
Substitute \eqref{eq:bellman_nonlinear} in \eqref{eq:Delta_rec_1} and use Poisson decomposition \(\xi_t = \xi_t^{(0)} + \xi_t^{(1)}\):
\begin{align}
     \Delta_{t+1} &= (I - \alpha_t D_\mu(I- \gamma\MKQ^{\pi^\star}))\Delta_t + \alpha_t\gamma D_{\mu}(\MKQ^{\pi_t} - \MKQ^{\pi^\star})Q_t + \alpha_t\xi_t^{(0)} + \alpha_t\xi_t^{(1)}\eqsp.
\end{align}
For simplicity denote \(G= D_\mu(I- \gamma\MKQ^{\pi^\star})\).
Moving \(G\Delta_t\) to the left-hand side:
\begin{equation}
    G\Delta_t = \alpha_t^{-1}(\Delta_t -\Delta_{t+1})+ \gamma D_{\mu}(\MKQ^{\pi_t} - \MKQ^{\pi^\star})Q_t + \xi_t^{(0)} + \xi_t^{(1)}\eqsp.
\end{equation}
Passing to Polyak-Ruppert averages and normalizing by \(\sqrt{t}\) yields:
\begin{align}\label{eq:clt_nonwheigted_decomposition}
    \sqrt{n}G\bar{\Delta}_n &=\frac{1}{\sqrt{n}}\sum_{t=1}^n\alpha_t^{-1}(\Delta_t -\Delta_{t+1}) + \frac{1}{\sqrt{n}}\sum_{t=1}^n \gamma D_{\mu}(\MKQ^{\pi_t} - \MKQ^{\pi^\star})Q_t + \frac{1}{\sqrt{n}}\sum_{t=1}^n \xi_t^{(1)}\\
    &+\frac{1}{\sqrt{n}}\sum_{t=1}^n\gamma (\pois{\bMKQ}{t} - \totMKQ \pois{\bMKQ}{t-1})(V_t - V^\star)- (\pois{\Lam}{t} - \totMKQ \pois{\Lam}{t-1})(Q_t-Q^\star)\\
    &+ \frac{1}{\sqrt{n}}\sum_{t=1}^n (\pois{\beps}{t} - \totMKQ\pois{\beps}{t-1})\\
    &= \frac{1}{\sqrt{n}}\sum_{t=1}^n (\pois{\beps}{t} - \totMKQ\pois{\beps}{t-1}) + \mathrm{R}_n^{\operatorname{pr}}\eqsp.
\end{align}
Here the remainder term $\mathrm{R}_n^{\operatorname{pr}}$is given by
    \begin{align}
         \mathrm{R}_n^{\operatorname{pr}} &=\underbrace{\frac{1}{\sqrt{n}}\sum_{t=1}^n\alpha_t^{-1}(\Delta_t -\Delta_{t+1})}_{\mathcal{T}_1} + \underbrace{\frac{1}{\sqrt{n}}\sum_{t=1}^n \gamma D_{\mu}(\MKQ^{\pi_t} - \MKQ^{\pi^\star})Q_t}_{\mathcal{T}_2} + \underbrace{\frac{1}{\sqrt{n}}\sum_{t=1}^n \xi_t^{(1)}}_{\mathcal{T}_3}\\
    &+\underbrace{\frac{1}{\sqrt{n}}\sum_{t=1}^n\gamma (\pois{\bMKQ}{t} - \totMKQ \pois{\bMKQ}{t-1})(V_t - V^\star)- (\pois{\Lam}{t} - \totMKQ \pois{\Lam}{t-1})(Q_t-Q^\star)}_{\mathcal{T}_4}\eqsp.
    \end{align}
\paragraph{Bound for $\mathcal{T}_1$.}
A telescoping argument gives the identity
\begin{align}\label{eq:T1_telescope}
\sum_{t=1}^{n}\alpha_t^{-1}\bigl(\Delta_t-\Delta_{t+1}\bigr)
&=
\alpha_1^{-1}\Delta_1 - \alpha_n^{-1}\Delta_{n+1}
+ \sum_{t=1}^{n-1}\bigl(\alpha_{t+1}^{-1}-\alpha_t^{-1}\bigr)\Delta_{t+1}
\eqsp .
\end{align}
Applying Minkowski's inequality to~\eqref{eq:T1_telescope} yields
\begin{align}\label{eq:T1_minkowski_1}
\PE^{1/p}\Big[
\Bigl\|
\sum_{t=1}^{n}\alpha_t^{-1}\bigl(\Delta_t-\Delta_{t+1}\bigr)
\Bigr\|_\infty^{p}
\Big]
&\le
\alpha_1^{-1}\PE^{1/p}\!\left[\|\Delta_1\|_\infty^{p}\right]
+
\alpha_n^{-1}\PE^{1/p}\!\left[\|\Delta_{n+1}\|_\infty^{p}\right] \\
&\quad
+
\sum_{t=1}^{n-1}\bigl(\alpha_{t+1}^{-1}-\alpha_t^{-1}\bigr)
\PE^{1/p}\!\left[\|\Delta_{t+1}\|_\infty^{p}\right]
\eqsp .
\end{align}
Under \Cref{assum:steps} the stepsizes satisfy
\begin{align}\label{eq:alpha_increment_bound}
\alpha_{t+1}^{-1}-\alpha_t^{-1}
&=
\frac{(t+1+k_0)^{\omega}-(t+k_0)^{\omega}}{c_0}
\le
\frac{\omega (t+k_0)^{\omega-1}}{c_0}
=
\frac{\omega}{t+k_0}\,\alpha_t^{-1}
\le
\frac{1}{t\alpha_t}
\eqsp.
\end{align}
Next, invoke the moment bound in \Cref{th:moments}:
\begin{align}
\PE^{1/p}\!\Big[
\Bigl\|
\sum_{t=1}^{n}\alpha_t^{-1}\bigl(\Delta_t-\Delta_{t+1}\bigr)
\Bigr\|_\infty^{p}
\Big]
&\lesssim
\alpha_1^{-1}\alpha_0^{1/2}\,(1-\gamma)^{-5/2}\mumin^{-3/2}\,p\,\taumix\,\log(2dn^2) \\
& +
\alpha_n^{-1}\alpha_n^{1/2}\,(1-\gamma)^{-5/2}\mumin^{-3/2}\,p\,\taumix\,\log(2dn^2) \\
& +
(1-\gamma)^{-5/2}\mumin^{-3/2}\,p\,\taumix\,\log(2dn^2)\sum_{t=1}^{n-1}(t\alpha_t)^{-1}
\alpha_t^{1/2}
\eqsp,
\end{align}
and factoring out the common term yields
\begin{align}\label{eq:T1_pre_simplify}
\PE^{1/p}\!\Big[
\Bigl\|
\sum_{t=1}^{n}\alpha_t^{-1}\bigl(\Delta_t-\Delta_{t+1}\bigr)
\Bigr\|_\infty^{p}
\Big]
&\lesssim
\frac{p\,\taumix\,\log(2dn^2)}{(1-\gamma)^{5/2}\mumin^{3/2}}
\Bigl(
\alpha_0^{-1/2} + \alpha_n^{-1/2} + \sum_{t=1}^{n-1}t^{-1}\alpha_t^{-1/2}
\Bigr)
\eqsp .
\end{align}
It remains to control the sum in \eqref{eq:T1_pre_simplify}
\begin{align}
\sum_{t=1}^{n-1}\frac{1}{t}\alpha_t^{-1/2}
&=
c_0^{-1/2}\sum_{t=1}^{n-1} t^{-1}(t+k_0)^{\omega/2}
\lesssim
c_0^{-1/2}k_0^{\omega/2}\sum_{t=1}^{n-1} t^{\omega/2-1}
\lesssim
\frac{c_0^{-1/2}\,n^{\omega/2}}{(1-\gamma)\mumin}
\eqsp.
\end{align}
Moreover, $\alpha_n^{-1/2}= c_0^{-1/2}(n+k_0)^{\omega/2}\lesssim (1-\gamma)^{-1}\mumin^{-1}n^{\omega/2}$.
Substituting these bounds into the definition of $\mathcal{T}_1$ and dividing by $\sqrt{n}$ yields
\begin{align}
\PE^{1/p}\!\left[\|\mathcal{T}_1\|_\infty^{p}\right]
\lesssim\frac{ n^{\frac{\omega}{2 } - \frac{1}{2}}p\,\taumix\,\log(2dn^2)}{(1-\gamma)^{7/2}\mumin^{5/2}}
\eqsp .
\end{align}

\paragraph{Bound for \(\mathcal{T}_2\).}
The inequalities established in~\eqref{eq:greedy_kernel} yield
    \(\MKQ^{\pi^\star}Q^\star \geq \MKQ^{\pi_t}Q^\star, \quad \MKQ^{\pi_t}Q_t \geq \MKQ^{\pi^\star}Q_t\eqsp.\)
Consequently,
\begin{align}
    0 \leq (\MKQ^{\pi_j} - \MKQ^{\pi^\star})Q_j &= (\MKQ^{\pi_j} - \MKQ^{\pi^\star})(Q_j - Q^\star) + Q^\star(\MKQ^{\pi_j} - \MKQ^{\pi^\star}) \\
    &\leq  (\MKQ^{\pi_j} - \MKQ^{\pi^\star})(Q_j - Q^\star)
\end{align}
where the last inequality uses $(\MKQ^{\pi_j} - \MKQ^{\pi^\star})Q^\star \le 0$.
As established in \cite{li2023statistical}[Lemma B.1], under \Cref{assum:gap} one has
\begin{equation}\label{eq:gap_bound}
    \supnorm{(\MKQ^{\pi_Q} - \MKQ^{\pi^\star})(Q - Q^\star)}
    \leq
    4\kappa_{\mathcal{M}}^{-1}\supnorm{Q - Q^\star}^2\eqsp.
\end{equation}
In particular, \eqref{eq:gap_bound} implies
\[
\bigl\|(\MKQ^{\pi_j} - \MKQ^{\pi^\star})Q_j\bigr\|_\infty
\le
\bigl\|(\MKQ^{\pi_j} - \MKQ^{\pi^\star})(Q_j - Q^\star)\bigr\|_\infty
\le
4\kappa_{\mathcal{M}}^{-1} \|\Delta_j\|_\infty^{2}
\eqsp .
\]
Using Minkowski's inequality yields
\begin{align}
    \E^{1/p}[\supnorm{\sum_{t=1}^n \gamma D_{\mu}(\MKQ^{\pi_t} - \MKQ^{\pi^\star})Q_t}^p] &\overset{(a)}{\lesssim }\kappa_{\mathcal{M}}^{-1} \sum_{t=1}^{n}\E^{1/p}[\supnorm{\Delta_t}^{2p}]\overset{(b)}{\lesssim}  \frac{p^2\taumix^2\log^2(2dn^2)}{\kappa_{\mathcal{M}}(1-\gamma)^{5}\mumin^{3} }\sum_{t=1}^{n}\alpha_t\\
    &\overset{(c)}{\lesssim} \frac{n^{1-\omega}p^2\taumix^2\log^2(2dn^2)}{(1-\omega)(1-\gamma)^6\mumin^4 \kappa_{\mathcal{M}}}\eqsp,
\end{align}
where (a) follows from \Cref{assum:gap} together with \eqref{eq:gap_bound},  
(b) follows from \Cref{th:moments} and (c) from \(k_0^{1-\omega}\asymp \frac{1}{(1-\gamma)\mumin}\) .
\paragraph{Bound for \(\mathcal{T}_3\).}
Using summation by parts, the full sum admits the exact decomposition
\begin{align}
\sum_{t=1}^n \xi_t^{(1)}
&=
\bigl(\totMKQ \pois{\varepsilon}{0}-\totMKQ \pois{\varepsilon}{n}\bigr)
-\gamma\bigl(\totMKQ \pois{\bMKQ}{n}(V_n-V^\star)-\totMKQ \pois{\bMKQ}{0}(V_0-V^\star)\bigr) \\
&\quad
+\bigl(\totMKQ \pois{\Lam}{n}(Q_n-Q^\star)-\totMKQ \pois{\Lam}{0}(Q_0-Q^\star)\bigr) \\
&\quad
+\gamma\sum_{t=1}^{n-1}\totMKQ \pois{\bMKQ}{t}\,(V_{t+1}-V_t)
-\sum_{t=1}^{n-1}\totMKQ \pois{\Lam}{t}\,(Q_{t+1}-Q_t)
\eqsp .
\end{align}
It was shown in the proof of \Cref{lem:xi_1_moments} that, for all $j$,
\begin{align}
    \E^{1/p}[\supnorm{V_{j+1} - V_{j}}^p] \lesssim\E^{1/p}[\supnorm{Q_{j+1} - Q_{j}}^p]
    &\lesssim \frac{\alpha_j}{1-\gamma} \eqsp.
\end{align}
By Minkowski's inequality and the uniform bounds on the Poisson terms,
\begin{align}
\PE^{1/p}\Big[\Big\|\sum_{t=1}^n \xi_t^{(1)}\Big\|_{\infty}\Big] \lesssim \frac{\taumix}{1-\gamma} + \frac{\taumix}{1-\gamma}\sum_{t=1}^n\alpha_j \lesssim\frac{n^{1-\omega}\taumix}{(1-\gamma)^2(1-\omega)\mumin}\eqsp.
\end{align}
Consequently, \[\PE^{1/p}[\supnorm{\mathcal{T}_3}^p]\lesssim \frac{n^{\frac{1}{2}-\omega}\taumix}{(1-\gamma)^2(1-\omega)\mumin}\eqsp.\]

\paragraph{Bound for \(\mathcal{T}_4\).} 
 Bound on the \(\mathcal{T}_4\) are derived in the same manner as in  
\Cref{lem:xi_0_moments}, by applying the concentration inequality in \Cref{lem:bdg_2} together with the moment estimates in \Cref{th:moments}
\begin{align}
    &\E^{1/p}[\supnorm{ \sum_{t=1}^n\gamma (\pois{\bMKQ}{t} - \totMKQ \pois{\bMKQ}{t-1})(V_t - V^\star)- (\pois{\Lam}{t} - \totMKQ \pois{\Lam}{t-1})(Q_t-Q^\star)}^p] \\
    &\qquad\lesssim p\taumix\log(2dn^2)\,\Big(\frac{1}{1-\gamma} + \Big(\sum_{t=1}^{n}\PE^{2/p}[\supnorm{Q_j-Q^\star}^p]\Big)^{1/2} \Big)\\
    &\qquad \lesssim \frac{p^2\taumix^2\log^2(2dn^2)}{(1-\gamma)^{5/2}\mumin^{3/2}}\,\Big(  \sum_{t=1}^{n}\alpha_j\Big)^{1/2}\\
    &\qquad \lesssim \frac{n^{\frac{1}{2}-\frac{\omega}{2}}p^2\taumix^2\log^2(2dn^2)}{(1-\omega)^{1/2}(1-\gamma)^{3}\mumin^{2} }  \eqsp.
\end{align}
\end{proof}

\begin{lemma}
\label{lem:pois_beps_l2}
Assume \Cref{assum:regularity} and \Cref{assum:UGE}. Then it holds that
\[
\|\pois{\beps}{}\|_2\leq\frac{\taumix}{1-\gamma}\eqsp.
\]
\end{lemma}
\begin{proof}
Recall that $\|\beps\|_\infty \le 2(1-\gamma)^{-1}$ and that $\beps(x)$ has at most one nonzero coordinate
for each $x\in\mathsf X$. Fix $u\in\rset^d$ with $\|u\|_2=1$ and define the scalar function
$\varphi_u(x) \coloneqq \langle u,\beps(x)\rangle$.
Then, for all $x\in\mathsf X$,
\[
|\varphi_u(x)|
\le \|u\|_2\,\|\beps(x)\|_2
\le 2(1-\gamma)^{-1}\eqsp.
\]
Applying \Cref{assum:UGE} together with \eqref{eq:markov:norm_bound_tmix} to $\varphi_u$ yields
\[
|\pois{\varphi_u}{}|
\le \tfrac{4}{3}\taumix\,\sup_{x\in\Xset}|\varphi_u(x)|
\le \tfrac{8}{3}\taumix\,(1-\gamma)^{-1}
\eqsp .
\]
By linearity, $\langle u,\pois{\beps}{}\rangle = \pois{\varphi_u}{}$, and therefore, by duality,
\[
\|\pois{\beps}{}\|_2
=
\sup_{\|u\|_2=1} \bigl|\langle u,\pois{\beps}{}\rangle\bigr|
=
\sup_{\|u\|_2=1} |\pois{\varphi_u}{}|
\le
\tfrac{8}{3}\taumix\,(1-\gamma)^{-1}
\eqsp .
\]
\end{proof}

\begin{lemma}
\label{lem:cov_nonstationary_taumix}
Assume \Cref{assum:regularity} and \Cref{assum:UGE}. Then it holds that
\[
\|\Sigmabf_{n}-\Sigmabf_{\varepsilon}\|_{\operatorname{ch}}
\lesssim
\frac{\|G^{-1}\|_{\infty}^2\taumix^3}{n(1-\gamma)^2}\eqsp.
\]
\end{lemma}

\begin{proof}
Let $\nu$ be an arbitrary initial distribution.
Recall that covariance at step \(n\) is given by
\[\Sigmabf_{n}= \frac{1}{n}\sum_{k=1}^{n}G^{-1}\E_{\nu}[(\pois{\beps}{k} - \totMKQ\pois{\beps}{k-1})(\pois{\beps}{k} - \totMKQ\pois{\beps}{k-1})^\top]G^{-\top}\eqsp.\]
Using the matrix inequality \(\|G^{-1}AG^{-\top}\|_{\operatorname{ch}}\leq \|A\|_{\operatorname{ch}}\|G^{-1}\|_{\infty}^2\) we obtain
\begin{align}
    \|\Sigmabf_n - \Sigmabf_{\infty}\|_{\operatorname{ch}}\leq \frac{\|G^{-1}\|_{\infty}^2}{n}\Big\|\sum_{k=1}^{n}\Big(\E_{\nu}[(\pois{\beps}{k} - \totMKQ\pois{\beps}{k-1})(\pois{\beps}{k} - \totMKQ\pois{\beps}{k-1})^\top] - \Sigmabf_{\beps}\Big)\Big\|_{\operatorname{ch}} \eqsp.
\end{align}
Fix \(1\leq i, j\leq d\) and define the scalar function \(\varphi_{ij}(x)  : \Xset \rightarrow \mathbb{R}\) 
\[
\varphi_{ij}(x):=\PE\left[\e_i(\pois{\beps}{k} - \totMKQ\pois{\beps}{k-1})(\pois{\beps}{k} - \totMKQ\pois{\beps}{k-1})^\top \e_j\mid X_{k-1}=x\right].
\]
By the uniform bounds on $\pois{\beps}{}$,
\[|\varphi_{ij}(x)|\lesssim\frac{\taumix^2}{(1-\gamma)^2}\eqsp.\]
\Cref{assum:UGE} therefore yields
\begin{align}\label{eq:uge_1}
    \sum_{k=1}^n |\varphi_{ij}(X_k) - \PE_{\bar\mu}[\varphi_{ij}]| &\leq \sup_{x\in\Xset}|\varphi_{ij}(x)| \sum_{i=k}^{n}(1/4)^{\lfloor k/\taumix\rfloor} \lesssim  \frac{\taumix^3}{(1-\gamma)^2}\eqsp.
\end{align}

\end{proof}

\begin{lemma}\label{lem:clt_application}
Assume \Cref{assum:regularity}--\Cref{assum:gap}. Then, for any $n$, the following bound holds:
\begin{align}
d_{K}(W_n,\,\Sigmabf_n^{1/2}Y)
&\lesssim
\frac{\log(d)\log^{1/2}(n)}{n^{1/2}}
\left(\frac{\|\Sigma_n\|}{\varsigma^{1/2}\lambda_{\min}^2(\Sigma_n)}\right)
\nonumber\\
&\quad+
\frac{\log(d)\log^{1/4}(n)}{n^{1/4}}
\left(
\frac{\overline{\sigma}^{1/2}(\Sigma_n)}{\underline{\sigma}^{1/2}(\Sigma_n)}
+
\frac{\|\Sigma_n\|_2^{1/2}}{\varsigma^{1/4}\underline{\sigma}(\Sigma_n)\lambda_{\min}^{1/2}(\Sigma_n)}
\right)
\nonumber\\
&\quad+
\frac{\log^{5/4}(d)\log(2d^2 n)}{n^{1/4}}
\frac{1}{\underline{\sigma}^{1/2}(\Sigma_n)}
\Bigg(\frac{\|G^{-1}\|_{\infty}\taumix}{1-\gamma}\Bigg)^{3/2}
\left\|\frac{\nu}{\mu}\right\|_{\infty}^{1/2}\Bigg(
\frac{1}{\lambda_{\min}(\Sigma_{\infty})}
+
\frac{1}{\varsigma\lambda_{\min}^3(\Sigma_{\infty})}
\Bigg)^{1/2}
\eqsp,
\end{align}
where \(\varsigma\) defined in \eqref{eq:varsigma_def}
\end{lemma}

\begin{proof}
The proof proceeds by applying \Cref{lem:clt_general}
to the martingale sum
\[
W_n = \sum_{k=1}^n X_k \eqsp,
\qquad
X_k = \frac{1}{\sqrt n}\,G^{-1}\bigl(\pois{\beps}{k}-\totMKQ\pois{\beps}{k-1}\bigr)\eqsp.
\]

\paragraph{Step 1: Uniform bounds on the martingale increments.}
Since $\beps$ has at most one nonzero coordinate,
\[
\|\beps\|_2=\|\beps\|_\infty \lesssim \frac{1}{1-\gamma}\eqsp.
\]
Moreover, \Cref{lem:pois_beps_l2} yields
\[
\|\pois{\beps}{}\|_2 \lesssim \frac{\taumix}{1-\gamma}\eqsp.
\]
Consequently,
\[
\|\totMKQ\pois{\beps}{}\|_2
=
\|\pois{\beps}{}-\beps\|_2
\lesssim
\frac{\taumix}{1-\gamma}\eqsp.
\]
These bounds imply
\begin{align}\label{eq:x_k_l2_bound}
\|X_k\|_2
&\lesssim
\frac{\|G^{-1}\|_2\taumix}{\sqrt n(1-\gamma)} \eqsp,
\qquad
\|X_kX_k^\top\|_2
=
\|X_k\|_2^2
\lesssim
\frac{\|G^{-1}\|_2^2\taumix^2}{n(1-\gamma)^2}\eqsp.
\end{align}

\paragraph{Step 2: Concentration of the predictable quadratic variation.}
Define
\[
F(\bar z_{k-1})
:=
\PE_{k-1}[X_kX_k^\top]-\PE[X_kX_k^\top]\eqsp.
\]
By \eqref{eq:x_k_l2_bound},
\[
\|F(\bar z_{k-1})\|
\lesssim
\frac{\|G^{-1}\|^2\taumix^2}{n^(1-\gamma)^2}
\eqsp.
\]
Since \Cref{assum:UGE} ensures that \(\{\bar z_k\}\) is an ergodic Markov chain with an absolute spectral gap \(\lambda\), the matrix Hoeffding inequality for Markov chains
(\cite{neeman2024concentration}[Theorem~2.5, Corollary~2.8]) yields
\begin{align}
\Pb\Big(
\Big\|
\sum_{k=1}^n \PE_{k-1}[X_kX_k^\top]-\Sigma_n
\Big\| \ge t
\Big)
\le
2d^{2-\pi/4}\exp\!\Big(
-\frac{(1-\gamma)^4}{\alpha(\lambda)\|G^{-1}\|^4\taumix^4}\,nt^2
\Big)\eqsp,
\end{align}
where \(\alpha(\lambda) = \frac{1+\lambda}{1-\lambda}\).
Thus, the concentration assumption \eqref{eq:concentration_assumption} in
\Cref{lem:clt_general} holds with
\begin{align}\label{eq:varsigma_def}
\varsigma = \frac{(1-\gamma)^4}{\alpha(\lambda)\|G^{-1}\|^4\taumix^4}\eqsp.
\end{align}

\paragraph{Step 3: Reduction to the eigenvalue sum.}
Applying \Cref{lem:clt_general} yields
\begin{align}\label{eq:clt_with_lambda}
d_{K}(W_n,\,\Sigmabf_n^{1/2}Y)
&\lesssim
\frac{[\ln_+ d]^{5/4}}{\underline{\sigma}^{1/2}(\Sigma_n)}
\Bigg(
\frac{\overline{\sigma}(\Sigma_n)}{\sqrt n}
+
\sum_{k=1}^n
\PE\frac{\|X_k\|_\infty^3}{\lambda_{\min}(P_k+\Sigma_n/n)}
\Bigg)^{1/2}
\nonumber\\
&\quad+
\frac{1}{(\varsigma n)^{1/2}}
\Bigg(
\frac{\log(d)\log^{1/2}(n)\|\Sigma_n\|}{\lambda_{\min}^2(\Sigma_n)}
\Bigg)
\nonumber\\
&\quad+
\frac{1}{(\varsigma n)^{1/4}}
\Bigg(
\frac{\log(d)\log^{1/4}(n)\|\Sigma_n\|_2^{1/2}}
{\underline{\sigma}(\Sigma_n)\lambda_{\min}^{1/2}(\Sigma_n)}
\Bigg).
\end{align}
Using \eqref{eq:x_k_l2_bound}, this further implies
\begin{align}
d_{K}(W_n,\,\Sigmabf_n^{1/2}Y)
&\lesssim
\frac{\log(d)\log^{1/2}(n)}{n^{1/2}}
\left(\frac{\|\Sigma_n\|}{
\varsigma^{1/2}\lambda_{\min}^2(\Sigma_n)}\right)
\nonumber\\
&\quad+
\frac{\log(d)\log^{1/4}(n)}{n^{1/4}}
\left(
\frac{\overline{\sigma}^{1/2}(\Sigma_n)}{\underline{\sigma}^{1/2}(\Sigma_n)}
+
\frac{\|\Sigma_n\|_2^{1/2}}
{\varsigma^{1/4}\underline{\sigma}(\Sigma_n)\lambda_{\min}^{1/2}(\Sigma_n)}
\right)
\nonumber\\
&\quad+
\frac{\log^{5/4}(d)}{\underline{\sigma}^{1/2}(\Sigma_n)}
\Bigg(
\frac{\|G^{-1}\|_{\infty}\taumix}{1-\gamma}
\Bigg)^{3/2}
\frac{1}{n^{3/4}}
\Bigg(
\sum_{k=1}^n \PE \lambda_{\min}^{-1}(P_k+\Sigma_n/n)
\Bigg)^{1/2}
\eqsp.
\end{align}
Therefore, it remains to control the sum of inverse minimal eigenvalues.

\paragraph{Step 4: Control of the eigenvalue sum.}
Let
\[
P_k := \sum_{i=k}^n \PE_{\nu}[X_iX_i^\top]\eqsp,
\]
where $\nu$ denotes the initial distribution of the underlying Markov chain.
The goal is to show that
\[
\sum_{k=1}^n \PE \lambda_{\min}^{-1}(P_k+\Sigma_\infty/n)
\lesssim
n\log n
+
\frac{n\log(2d^2n)}{\varsigma\lambda_{\min}^2(\Sigma_\infty)}\eqsp.
\]
The argument is based on a decomposition into two regimes. Introduce a burn-in
index
\begin{align}\label{eq:burn_in_definition}
n_0 := \frac{\log(2d^2n)}{\varsigma\lambda_{\min}^2(\Sigma_\infty)}\eqsp.
\end{align}
The sum is split into the ranges
\[
1\le k\le n-n_0
\qquad\text{and}\qquad
n-n_0<k\le n\eqsp.
\]
We first control the nonterminal range. By
\cite{neeman2024concentration}[Theorem~2.5, Corollary~2.8],
\begin{align}\label{eq:concentration_P_k}
\Pb\Big(
\Big\|
\sum_{\ell=k}^n \PE_{\ell-1}[X_\ell X_\ell^\top]-\PE[P_k]
\Big\| \ge t
\Big)
\le
2d^2\exp\!\Big(
-\frac{\varsigma n^2 t^2}{n-k+1}
\Big)\eqsp.
\end{align}
Since, for any nonnegative functional $H$,
\[
\PE_{z_1\sim\nu}[H(z_1,\ldots,z_n)]
\le
\Big\|\frac{\nu}{\mu}\Big\|_\infty
\PE_{z_1\sim\mu}[H(z_1,\ldots,z_n)]\eqsp,
\]
it suffices to consider the stationary initialization $z_1\sim\mu$.
In that case,
\[
\PE[P_k] = \frac{n-k+1}{n}\Sigma_\infty\eqsp.
\]
Rewriting \eqref{eq:concentration_P_k} in terms of $\delta$, we obtain that,
with probability at least $1-\delta$,
\[
\Big\|P_k-\frac{n-k+1}{n}\Sigma_\infty\Big\|
\le
\frac{\sqrt{n-k+1}\,\log^{1/2}(2d^2/\delta)}{\sqrt{\varsigma}\,n}\eqsp.
\]
Hence, by Lidskii's inequality,
\[
\lambda_{\min}(P_k+\Sigma_\infty/n)
\ge
\frac{n-k+1}{n}\lambda_{\min}(\Sigma_\infty)
-
\frac{\sqrt{n-k+1}\,\log^{1/2}(2d^2/\delta)}{\sqrt{\varsigma}\,n}
\]
on the same event. On the complementary event, the crude bound
\[
\lambda_{\min}(P_k+\Sigma_\infty/n)\ge \frac{1}{n}\lambda_{\min}(\Sigma_\infty)
\]
always holds. Choosing $\delta=(n-k+1)^{-1}$ gives
\begin{align}
&\PE\!\left[\frac{1}{\lambda_{\min}(P_k+\Sigma_\infty/n)}\right]\le
\frac{\delta n}{\lambda_{\min}(\Sigma_\infty)}
+
\frac{n}
{(n-k+1)\lambda_{\min}(\Sigma_\infty)
-
\sqrt{n-k+1}\,\frac{\log^{1/2}(2d^2/\delta)}{\sqrt{\varsigma}}}
\nonumber\\
&\qquad\le
\frac{n}{(n-k+1)\lambda_{\min}(\Sigma_\infty)}
+
\frac{n}
{(n-k+1)\lambda_{\min}(\Sigma_\infty)
-
\sqrt{n-k+1}\,\frac{\log^{1/2}(2d^2(n-k+1))}{\sqrt{\varsigma}}}
\nonumber\\
&\qquad\le
\frac{3n}{(n-k+1)\lambda_{\min}(\Sigma_\infty)}
\end{align}
provided that $k\le n-n_0$. Summing over this range gives
\begin{align}\label{eq:start_range_bound}
\sum_{k=1}^{n-n_0}
\PE\!\left[\frac{1}{\lambda_{\min}(P_k+\Sigma_\infty/n)}\right]
\lesssim
\frac{n\log n}{\lambda_{\min}(\Sigma_\infty)}\eqsp.
\end{align}
It remains to control the terminal range. Using the crude lower bound
\(
\lambda_{\min}(P_k+\Sigma_\infty/n)\ge n^{-1}\lambda_{\min}(\Sigma_\infty)
\),
we obtain
\begin{align}\label{eq:burn_in_range}
\sum_{k=n-n_0}^{n}
\PE\!\left[\frac{1}{\lambda_{\min}(P_k+\Sigma_\infty/n)}\right]
\lesssim
\frac{n_0 n}{\lambda_{\min}(\Sigma_\infty)}
\lesssim
\frac{n\log(2d^2 n)}{\varsigma\lambda_{\min}^3(\Sigma_\infty)}
\eqsp.
\end{align}
Combining \eqref{eq:start_range_bound} and \eqref{eq:burn_in_range} and
substituting the resulting estimate into \eqref{eq:clt_with_lambda} yields the desired bound.
\end{proof}

%% file: appendix_martingale_limits.tex
\subsection{Gaussian approximation}
\label{section:gar}
\textbf{Notations and definitions.} In this section, we provide the proofs of results corresponding to the central limit theorem for vector-valued martingale difference sequences. We follow the notations outlined in \Cref{sec:vector-martingales-bounds} under \Cref{assum:assumptions}. Namely, we assume that $\{X_k\}_{k\ge 1}$ is a $d$-dimensional martingale difference w.r.t. filtration $\{\mathcal{F}_k\}_{k\ge 0}$. Recall that we write $V_k$ and \(P_k\)  for 
\[
V_k = \PE[X_k X_k^\top \mid \mathcal{F}_{k-1}]\eqsp,\quad P_k = \sum_{i=k}^nV_i\eqsp.
\]

Consider a random vector $\eta \sim \mathcal{N}(0,I_d)$, independent of
$\mathcal{F}_n$. We approximate the probabilities with the following smooth function:
\[
\varphi_r(x,\varepsilon) := \mathbb{P}\bigl(x + \varepsilon \eta \in A_r\bigr)\eqsp,\quad A_r \in \mathcal{R}\eqsp,
\]
where
$\varepsilon > 0$ is a fixed positive number, and $\mathcal R$ is a class of hyperrectangles. Note that for a fixed $\varepsilon$, the function
$x \mapsto \varphi_r(x,\varepsilon)$ is infinitely differentiable. Furthemore, the following lemma holds. 
\begin{lemma}[Lemma~2.3 in \cite{10.1214/20-AAP1629}]
\label{sum of derivatives lemma}
For each $x,r \in \mathbb{R}^d$ and integer $s \ge 1$
we have
\begin{equation}
\label{eq:phi-derivative-bound}
\sum_{j_1,\ldots,j_s=1}^{d}
\left|
\frac{\partial}{\partial x_{j_1}}
\cdots
\frac{\partial}{\partial x_{j_s}}
\varphi_r(x,\varepsilon)
\right|
\le
C_s \varepsilon^{-s} (\log d)^{s/2}\eqsp,
\end{equation} 
where \(C_s>0\) is a constant depending only on \(s\).
\end{lemma}
For convenience, we restate \Cref{general CLT for martingales with fixed qc}.

\begin{theorem}\label{thm:clt_const_var}
Under \Cref{assum:assumptions} and \Cref{assum:const_variation}, for any \(d\)-dimensional symmetric positive-definite matrix \(\Sigma\succ 0 \), 
\begin{equation}
d_{K}(S,T)\lesssim
\frac{[\ln_+ d]^{5/4}}{\underline{\sigma}^{1/2}(\Sigma_n)}
\Bigg(\overline{\sigma}(\Sigma)+\sum_{k}\PE\frac{\|X_k\|_\infty^3 }{\lambda_{\min}({P}_k +\Sigma)}\Bigg)^{1/2}\eqsp.
\end{equation}
\end{theorem}
\begin{proof}
Throughout the proof, the parameter $\varepsilon > 0$ is fixed, and will be optimized at the end.
For $k \in \{0,\dots,n\}$, define the partial sums
\[
S_k \coloneqq \sum_{i=1}^k X_i\eqsp,
\qquad
T_k \coloneqq \sum_{i=k}^n V_i^{1/2} Z_i\eqsp,
\]
 where \(Z_i\) are i.i.d. \(\mathcal{N}(0,I)\) random vectors, independent of the
martingale difference sequence.
We use the conventions $S_0 = 0$ and $T_{n+1} = 0$. Also denote shifted random vectors and covariances as
\begin{equation}
    \bar T_k=T_k + \Sigma^{1/2}Z\eqsp, \quad \bar P_k = P_k+\Sigma\eqsp,
\end{equation}
where Z is independent standard gaussian random variable independent of the filtration \(\{\mathcal{F}_k\}_{k=0}^n\) and \(Z_i\).
We start with applying Lemma \ref{lem:smoothing}:
\begin{align}
    d_{K}(S,T)
&\le
\sup_{r \in \mathbb{R}^d}
\bigl|
\PE\bigl[\varphi_r(S_n,\varepsilon)
-
\varphi_r(T_1,\varepsilon)]
\bigr|
+
C \varepsilon [\ln_{+} d]\underline{\sigma}^{-1}(\Sigma_n)\\
&\leq 
\sup_{r \in \mathbb{R}^d}
\bigl|
\PE\bigl[\varphi_r(S_n,\varepsilon)
-
\varphi_r(\bar T_1,\varepsilon)]
\bigr| + \sup_{r \in \mathbb{R}^d}
\bigl|
\PE\bigl[\varphi_r(\bar T_1,\varepsilon)
-
\varphi_r(T_1,\varepsilon)]
\bigr|
+
C \varepsilon [\ln_{+} d]\underline{\sigma}^{-1}(\Sigma_n).
\end{align}
The Lagrange mean value theorem combined with Lemma \ref{sum of derivatives lemma} implies that \(\varphi(x, \varepsilon)\) Lipshitz in the \(\infty\)-norm:
\begin{align}
    \|\varphi_r(x, \varepsilon) - \varphi_r(y,\varepsilon)\|_\infty &= \|\langle\nabla\varphi_r(x + \theta(y-x),\varepsilon), y-x\rangle\|_\infty \leq \|\nabla\varphi_r(x + \theta(y-x),\varepsilon)\|_1\|y-x\|_\infty\\
    &\leq C_1\varepsilon^{-1}[\ln_+d]^{1/2}\|x-y\|_\infty\eqsp.
\end{align}
Then, 
\begin{align}
    \sup_{r \in \mathbb{R}^d}
\bigl|
\PE\bigl[\varphi_r(\bar T_1,\varepsilon)
-
\varphi_r(T_1,\varepsilon)]
\bigr|&\lesssim 
\varepsilon^{-1}[\ln_+d]^{1/2}
\PE[\|\bar T_1 - T_1\|_\infty]
=
\varepsilon^{-1}[\ln_+d]^{1/2}
\PE[\|\Sigma^{1/2}Z\|_\infty]
\\
&\lesssim 
\varepsilon^{-1}[\ln_+d]^{1/2}\overline{\sigma}(\Sigma)\eqsp,
\end{align}
where \(\overline{\sigma}^2(\Sigma) = \max_j\Sigma_{jj}\).
Applying Lindeberg’s decomposition together with the triangle inequality yield the telescoping representation
\begin{equation}\label{eq:telescope}
\Bigl|
\mathbb{E}\bigl[\varphi_r(S_n,\varepsilon)
-
\varphi_r(\bar T_1,\varepsilon)\bigr]
\Bigr|
\le
\sum_{k=1}^n
\Bigl|
\mathbb{E}\bigl[
\varphi_r(S_k + \bar T_{k+1},\varepsilon)
-
\varphi_r(S_{k-1} + \bar T_k,\varepsilon)
\bigr]
\Bigr|.
\end{equation}
Lemma~\ref{lem:stein-solution} ensures that for each $k \in [n]$ there exists a \(\mathcal{F}_{k-1}\)-measurable function
$f_k : \mathbb{R}^d \to \mathbb{R}$ such that, for all $w \in \mathbb{R}^d$,
\begin{align}
\varphi_r(S_{k-1} + w, \varepsilon)
-
\mathbb{E}\!\left[\varphi_r\!\left(S_{k-1} + {\bar P_k}^{1/2} Z, \varepsilon\right)\right]
&=
\nabla^\top \bar P_k \nabla f_k(w)
-
w^\top \nabla f_k(w),
\label{eq:stein_fk}
\end{align}
where $Z \sim \mathcal{N}(0, I_d)$ is a standard $d$-dimensional Gaussian random vector.
Note that, by \Cref{assum:const_variation},
\begin{align}
P_{k+1} = P_1 - \sum_{j=1}^{k} V_j = \Sigma_n - \sum_{j=1}^k V_j\eqsp,\quad
P_{k} = P_1 - \sum_{j=1}^{k-1} V_j = \Sigma_n - \sum_{j=1}^{k-1} V_j\eqsp,
\end{align}
and hence both $P_k$ and $P_{k+1}$ are $\mathcal{F}_{k-1}$-measurable.
Each term in \eqref{eq:telescope} can then be rewritten as follows:
\begin{align*}
&\PE[\varphi_r(S_{k-1} + X_k + \bar T_{k+1} ,\varepsilon) \mid \mathcal F_{k-1}] 
-
\PE[\varphi_r(S_{k-1} + \bar T_k,\varepsilon) \mid \mathcal F_{k-1}] \\ 
&=\PE[\varphi_r(S_{k-1} + X_k + \bar T_{k+1} ,\varepsilon) \mid \mathcal F_{k-1}] 
-
\PE[\varphi_r(S_{k-1} + {\bar P_k}^{1/2} Z,\varepsilon) \mid \mathcal F_{k-1}] \\
&  = \PE[\nabla^\top \bar P_k \nabla f_k(X_k + \bar P_{k+1}^{1/2} Z)  - (X_k + \bar P_{k+1}^{1/2} Z)^\top\nabla f_k(X_k +\bar P_{k+1}^{1/2} Z) \mid \mathcal F_{k-1}]\eqsp,
\end{align*}
where the last equality follows from the fact that \(\bar T_{k+1}\mid\mathcal{F}_{k-1}\sim\mathcal{N}(0, \bar P_{k+1})\).
Using this identity and rearranging the terms, we obtain
\begin{align*}
&\PE[\varphi_r(S_{k-1} + X_k + \bar T_{k+1} ,\varepsilon) \mid \mathcal F_{k-1}]
-
\PE[\varphi_r(S_{k-1} + \bar T_k,\varepsilon) \mid \mathcal F_{k-1}] \\ 
&  = \PE[\nabla^\top \bar P_k \nabla f_k(X_k + \bar P_{k+1}^{1/2} Z)  \mid \mathcal F_{k-1}] - \PE[(X_k + \bar P_{k+1}^{1/2} Z)^\top\nabla f_k(X_k + \bar P_{k+1}^{1/2} Z) \mid \mathcal F_{k-1}] \\
& = \PE[\nabla^\top \bar P_{k+1} \nabla f_k(X_k + \bar P_{k+1}^{1/2} Z)  \mid \mathcal F_{k-1}] - \PE[(\bar P_{k+1}^{1/2} Z)^\top\nabla f_k(X_k + \bar P_{k+1}^{1/2} Z) \mid \mathcal F_{k-1}]\\ 
&+ \PE[\nabla^\top V_{k} \nabla f_k(X_k +\bar P_{k+1}^{1/2} Z)  \mid \mathcal F_{k-1}] - \PE[X_k^\top\nabla f_k(X_k + \bar P_{k+1}^{1/2} Z) \mid \mathcal F_{k-1}]
\end{align*}
By Lemma~\ref{lem:stein_equation},
$$
\PE[\nabla^\top \bar P_{k+1} \nabla f_k(X_k +\bar P_{k+1}^{1/2} Z)  \mid \mathcal F_{k-1}] - \PE[(\bar P_{k+1}^{1/2} Z)^\top\nabla f_k(X_k + \bar P_{k+1}^{1/2} Z) \mid \mathcal F_{k-1}] = 0\eqsp.
$$
Hence, 
\begin{align}\label{eq:lineberg_equality}
\Delta_{k} &:= \PE[\varphi_r(S_{k-1} + X_k +\bar T_{k+1} ,\varepsilon) \mid \mathcal F_{k-1}]  
-
\PE[\varphi_r(S_{k-1} + \bar T_{k},\varepsilon) \mid \mathcal F_{k-1}] \\ 
& =  \underbrace{\PE[\nabla^\top V_{k} \nabla f_k(X_k + \bar P_{k+1}^{1/2} Z)  \mid \mathcal F_{k-1}]}_{I_1} - \underbrace{\PE[X_k^\top\nabla f_k(X_k +\bar P_{k+1}^{1/2} Z) \mid \mathcal F_{k-1}]}_{I_2} = I_1 - I_2
\end{align}
We first decompose the term $I_1$ as follows:
\begin{align}
    I_1 = \PE_{k-1}\left[\operatorname{Tr}(V_k\nabla^2f_k(X_k + \bar P_{k+1}^{1/2} Z)) - \operatorname{Tr}(V_k\nabla^2f_k(\bar P_{k+1}^{1/2} Z))\right] +  \PE_{k-1}\left[\operatorname{Tr}(V_k\nabla^2f_k(\bar P_{k+1}^{1/2} Z))\right]
\end{align}
Next, using the martingale property together with the Taylor's formula we obtain
\begin{align*} 
I_2 &= \PE_{k-1}\left[X_k^\top \nabla f_k(X_k +  \bar P_{k+1}^{1/2} Z)\right]\\
&=\PE_{k-1}\left[X_k^\top \nabla f_k(  \bar P_{k+1}^{1/2} Z)
+\PE_\tau[X_k^\top \nabla^2 f_k(\tau X_k + \bar  P_{k+1}^{1/2} Z)X_k]\right]\\
&=\PE_\tau \PE_{k-1}\left[X_k^\top \nabla^2 f_k(\tau X_k + \bar  P_{k+1}^{1/2} Z)X_k\right] \\
&= \PE_\tau \PE_{k-1}\left[X_k^\top \left(\nabla^2 f_k(\tau X_k +\bar   P_{k+1}^{1/2} Z) - \nabla^2 f_k(  \bar P_{k+1}^{1/2} Z)  \right)X_k\right] \\
&+  \PE_{k-1}\left[X_k^\top  \nabla^2 f_k( \bar P_{k+1}^{1/2} Z)  X_k\right]\eqsp,
\end{align*}
where $\tau \in [0,1]$ is a uniformly distributed r.v. independent of $(\mathcal F_k)_{k \geq 1}, (Z_k)_{k \geq 1}$. 
By the definition of $V_k$ and conditional independence between \(X_k\) and \(Z\), we have the identity
\begin{align}\label{eq:first_order_equality}
\PE_{k-1}\left[X_k^\top \nabla^2 f_k(\bar  P_{k+1}^{1/2}Z)X_k\right]
&=
\operatorname{Tr}\bigl(
\PE_{k-1}[X_k X_k^\top]\PE_{k-1}[\nabla^2 f_k(\bar P_{k+1}^{1/2}Z)]
\bigr)
\\
&=
\PE_{k-1}[\operatorname{Tr}\bigl(V_k \nabla^2 f_k( \bar P_{k+1}^{1/2}Z)\bigr)].
\end{align}
Hence, by \eqref{eq:first_order_equality}, the linearity of the trace operator, and the $\mathcal{F}_{k-1}$-measurability of $V_k$, we obtain
\begin{align}\label{eq:second_order_diff}
    \Delta_k &= \operatorname{Tr}\left\{V_k\PE_{k-1}\left[\nabla^2f_k(X_k + \bar P_{k+1}^{1/2} Z) - \nabla^2f_k(\bar P_{k+1}^{1/2} Z)\right]\right\} \\
    &+ \PE_{k-1}\left[X_k^\top \left(\nabla^2 f_k(\tau X_k + \bar P_{k+1}^{1/2} Z) - \nabla^2 f_k(\bar  P_{k+1}^{1/2} Z)  
    \right)X_k\right]
\end{align}
Lemma~\ref{lem:stein-solution} guarantees the existence of a function $\tilde f_k$
satisfying the isotropic Stein equation
\begin{align}
\Delta \tilde f_k (w) - w^\top \nabla \tilde f_k (w)
=
\varphi_r(S_{k-1} + \bar P_k^{1/2} w, \varepsilon)
-
\PE[\varphi_r(S_{k-1} + \bar P_k^{1/2} Z, \varepsilon)]\eqsp.
\end{align}
As established in \eqref{eq:algebraic_relation}, the corresponding solution $f_k$
to the non--isotropic equation is given by
\begin{equation}
    f_k(w) = \tilde f_k(\bar P_k^{-1/2} w)\eqsp.
\end{equation}
Using this representation, we can rewrite $\Delta_k$ in terms of $\tilde f_k$ as
\begin{align}\label{eq:Delta_final_identity}
    \Delta_k
    &=
    \operatorname{Tr}\left\{
    \bar P_k^{-1/2} V_k \bar P_k^{-1/2}
    \PE_{k-1}\left[
    \nabla^2 \tilde f_k(\bar P_k^{-1/2}X_k + \widetilde{Z})
    -
    \nabla^2 \tilde f_k(\widetilde{Z})
    \right]
    \right\} \\
    &\quad
    +
    \PE_{\tau}\PE_{k-1}\left[
    (\bar P_k^{-1/2} X_k)^\top
    \left(
    \nabla^2 \tilde f_k(\tau \bar P_k^{-1/2} X_k + \widetilde{Z})
    -
    \nabla^2 \tilde f_k(\widetilde{Z})
    \right)
    (\bar P_k^{-1/2} X_k)
    \right],
\end{align}
where \(\widetilde{Z}\sim\mathcal{N}(0, \bar P_k^{-1/2}\bar P_{k+1}\bar P_k^{-1/2})\). By Lemma~\ref{lem:elemetwise_regularity} with \(\Sigma=\bar {P}_k\) we obtain
\begin{align}
    \left\|
    {\bar P}_k^{-1/2}\nabla^2 \tilde f_k (\tau  \bar P_k^{-1/2} X_k + \widetilde{Z})
    -
    \nabla^2 \tilde f_k(\widetilde{Z})\bar {P}_k^{-1/2}
    \right\|_e
\lesssim \varepsilon^{-1}\|X_k\|_\infty\log^{3/2}(d)\lambda_{\operatorname{min}}^{-1}({\bar P}_k)\eqsp.
\end{align}
Consequently, the second term in \eqref{eq:Delta_final_identity} can be bounded as
\begin{align}
&\left|\PE_{\tau}\PE_{k-1}\left[
    (\bar P_k^{-1/2} X_k)^\top
    \left(
    \nabla^2 \tilde f_k(\tau \bar P_k^{-1/2} X_k + \widetilde{Z})
    -
    \nabla^2 \tilde f_k(\widetilde{Z})
    \right)
    (\bar P_k^{-1/2} X_k)
    \right]\right|\\
    &\lesssim\PE_{\tau} \PE_{k-1}\left[ \|X_k\|_\infty^2\| \bar{P}_k^{-1/2}(\nabla^2 \tilde f_k(\tau \bar P_k^{-1/2} X_k + \widetilde{Z})
    -
    \nabla^2 \tilde f_k(\widetilde{Z}))\bar{P}_k^{-1/2}\|_e\right] \\
    &\lesssim \varepsilon^{-1}[\ln_+d]^{3/2}\lambda_{\operatorname{min}}^{-1}(\bar{P}_k)\PE_{k-1}[\|X_k\|_\infty^3]
\end{align}
Let \(\|\cdot\|_{\operatorname{ch}}\) denote the Chebyshev norm in \(\mathbb{R}^{d\times d}\) , i.e., for a \(d\times d\) matrix \(A\), \(\|A\|_{\operatorname{ch}} = \max_{i,j}|a_{ij}|\).
Meanwhile, the second term can also be bounded by 
\begin{align}
     &\left|\operatorname{Tr}\left\{
    V_k
    \PE_{k-1}\left[ \bar P_k^{-1/2}(
    \nabla^2 \tilde f_k(\bar P_k^{-1/2}X_k + \widetilde{Z})
    -
    \nabla^2 \tilde f_k(\widetilde{Z})) \bar P_k^{-1/2}
    \right]
    \right\}\right|\\
    &\quad \lesssim \| V_k \|_{\operatorname{ch}} \:\PE_{k-1}[\| \bar P_k^{-1/2}(\nabla^2 \tilde f_k(\bar P_k^{-1/2}X_k + \widetilde{Z})
    -
    \nabla^2 \tilde f_k(\widetilde{Z})) \bar P_k^{-1/2}\|_e]\\
    &\quad \lesssim \varepsilon^{-1}[\ln_+d]^{3/2}\lambda_{\operatorname{min}}^{-1}(\bar{P}_k)\PE_{k-1}[\|X_k\|_\infty^2] \PE_{k-1}[\|X_k\|_\infty] \\
    &\quad \lesssim  \varepsilon^{-1}[\ln_+d]^{3/2}\lambda_{\operatorname{min}}^{-1}(\bar{P}_k)\,\PE_{k-1}[\| X_k\|_\infty^3]
\end{align}
Summing over $k$ yields
\begin{align}
    d_{K}(S,T)
&\lesssim
\sum_{k=1}^n\PE[|\Delta_k|] + \varepsilon^{-1}[\ln_+d]^{1/2}\overline{\sigma}(\Sigma)
+
 \varepsilon [\ln_{+} d]\underline{\sigma}^{-1}(\Sigma_n)\\
&\lesssim \varepsilon^{-1}[\ln_+d]^{3/2}\Big(\overline{\sigma}(\Sigma)+\sum_{k=1}^n\PE[\lambda_{\operatorname{min}}^{-1}(\bar{P}_k)\| X_k\|_\infty^3]\Big)+ \varepsilon [\ln_{+} d]\underline{\sigma}^{-1}(\Sigma_n)
\end{align}
Optimizing the right-hand side over $\varepsilon > 0$ gives
\begin{align}
d_{K}(S,T)\lesssim
\,[\ln_+ d]^{5/4}\,\underline{\sigma}^{-1/2}(\Sigma_n)\,
\Big(\overline{\sigma}(\Sigma)+\sum_{k}\PE[\lambda_{\min}^{-1}(\bar{P}_k)\|X_k\|_\infty^3 ]\Big)^{1/2}\eqsp.
\end{align}
Substituting back $\bar P_k = P_k + \Sigma$ completes the proof.
\end{proof}

\paragraph{Auxiliary lemmas}

Recall the notation $\underline{\sigma}^2(\Sigma_n) =  \min_{1 \le j \le d} \Sigma_n(j,j)$, which denotes the minimal variance of $\Sigma_n$.

\begin{lemma}[Lemma 1 in \cite{2020arXiv200913673K}]
\label{lem:smoothing} Let \(T_1\sim\mathcal{N}(0, \Sigma_n)\).
Suppose that $\underline{\sigma}^2(\Sigma_n)> 0$. There exists a universal constant $C>0$ such that
for any $\varepsilon > 0$,
\begin{equation}\label{eq:kuchibolta}
d_{K}(S_n,T_1)
\le
\sup_{r \in \mathbb{R}^d}
\bigl|
\PE\bigl[\varphi_r(S_n,\varepsilon)
-
\varphi_r(T_1,\varepsilon)]
\bigr|
+
\frac{C \varepsilon \log(d)}{\underline{\sigma}(\Sigma_n)}\eqsp.
\end{equation}
\end{lemma}

Another key technical tool is Stein’s equation.
\begin{lemma}[Lemma 2.6 in \cite{10.1214/09-AOP467}]
\label{lem:stein-solution}
Let $h: \rset^d \to \rset$ be differentiable with bounded first derivative, \(Z\sim \mathcal{N}(0,I_d)\).
Then, if $\Sigma \in \rset^{d\times d}$ is symmetric and positive definite,
there exists a solution $f:\rset^d \to \rset$ to the equation
\begin{equation}\label{eq:stein}
\nabla^\top\Sigma \nabla f(w) - w^\top\nabla f(w)
= h(w) - \PE [h(\Sigma^{1/2}Z)],
\end{equation}
which holds for every $w \in \rset^d$.
If, in addition, $h$ is $n$ times differentiable, there exists a solution $f$
which is also $n$ times differentiable and we have for every $k=1,\dots,n$, the bound
\begin{equation}\label{eq:stein-bound}
\left|
\frac{\partial^{k} f(w)}{\prod_{j=1}^{k}\partial w_{i_j}}
\right|
\le
\frac{1}{k}
\left|
\frac{\partial^{k} h(w)}{\prod_{j=1}^{k}\partial w_{i_j}}
\right|,
\end{equation}
for every $w \in \rset^d$.
\end{lemma}
\begin{lemma}[Stein's identity]\label{lem:stein_equation}
Let $Z \sim \mathcal{N}(0,\Sigma)$ be a centered Gaussian random vector in $\mathbb{R}^d$ with covariance matrix $\Sigma$, and let $h:\mathbb{R}^d \to \mathbb{R}$ be a twice continuously differentiable function such that the expectations below are well defined. Then
\begin{align}
\mathbb{E}\!\left[ Z^\top \nabla h(Z) \right]
=
\mathbb{E}\!\left[ \operatorname{tr}\!\left( \Sigma \nabla^2 h(Z) \right) \right].
\end{align}
\end{lemma}
the following
corollary can be obtained by optimizing over $\Sigma$.

\paragraph{Change-of-Variables Formula for Stein Solutions}

We first establish an explicit algebraic relation between the solutions of the Stein equation
in the non--isotropic and isotropic settings.
Let $\tilde f_h$ be a solution to the isotropic Stein equation
\begin{align}
\Delta \tilde f_h (w) - w^\top \nabla \tilde f_h (w)
=
h(\Sigma^{1/2} w) - \PE[h(\Sigma^{1/2} Z)]\eqsp,
\end{align}
where $Z \sim \mathcal{N}(0, I_d)$.
Define the function $ f_h : \mathbb{R}^d \to \mathbb{R}$ by
\begin{equation}\label{eq:algebraic_relation}
 f_h(w) := \tilde f_h(\Sigma^{-1/2} w)\eqsp.
\end{equation}
We now verify that $ f_h$ solves the non--isotropic Stein equation
\begin{align}
\nabla^\top \Sigma \nabla  f_h(w) - w^\top \nabla  f_h(w)
=
h(w) - \PE[h(\Sigma^{1/2} Z)]\eqsp.
\end{align}
By the chain rule, we have
\begin{align}
\nabla  f_h(w)
&=
\Sigma^{-1/2} \nabla \tilde f_h(\Sigma^{-1/2} w)\eqsp,\\
\nabla^\top \Sigma \nabla  f_h(w)
&=
\operatorname{Tr}\!\left(
\Sigma \nabla^2 \bigl[\tilde f_h(\Sigma^{-1/2} w)\bigr]
\right)
=
\Delta \tilde f_h(\Sigma^{-1/2} w)\eqsp.
\end{align}
Substituting these expressions into the left-hand side of the isotropic Stein equation yields
\begin{align}
\nabla^\top \Sigma \nabla  f_h(w) - w^\top \nabla  f_h(w)
&=
\Delta \tilde f_h(\Sigma^{-1/2} w)
-
(\Sigma^{-1/2} w)^\top \nabla \tilde f_h(\Sigma^{-1/2} w) \\
&=
h(\Sigma^{1/2} \Sigma^{-1/2} w) - \PE[h(\Sigma^{1/2} Z)] \\
&=
h(w) - \PE[h(\Sigma^{1/2} Z)]\eqsp.
\end{align}
This proves the algebraic equivalence \eqref{eq:algebraic_relation}.

\paragraph{Elementwise Hessian Bounds for Smoothed Stein Solutions}
Let \(\|\cdot\|_e\) denote the element-wise norm in \(\mathbb{R}^{d\times d}\) , i.e., for a \(d\times d\) matrix \(A\), \(\|A\|_e = \sum_{i,j}|a_{ij}|\). Regularity properties of solutions to Stein’s equation for classes of H\"older functions
have been established in \cite{gallouet_mijoule_swan_stein}.
In what follows, we derive an improved result in the $\|\cdot\|_e$ norm for the special
case of the smoothing function $\varphi_r(x,\varepsilon)$. We also define 
\begin{equation}
    \varphi_r(x,\Sigma^{1/2}) = \mathbb{P}(x + \Sigma^{1/2}\eta\in A_r)\eqsp,\quad \eta\sim\mathcal{N}(0, I_d)\eqsp.
\end{equation}
With this notation \(\varphi(x,\varepsilon) = \varphi(x, \varepsilon I)\).
\begin{proposition}\label{lem:elemetwise_regularity}
    Let \(\Sigma\succ 0 \in\mathbb{S}^{d\times d}\) and \(\mu\in\mathbb{R}^d\). Further, define \(g:\mathbb{R}^d\xrightarrow{}\mathbb{R}\) as 
    \begin{equation}
        g(x) = \varphi_r(\mu + \Sigma^{1/2}x,\varepsilon)\eqsp,
    \end{equation}
    and use \(f_g\) to denote the solution to Stein’s equation
    \begin{equation}
        \Delta f_g(x) - x^\top\nabla f_g(x) = g(x) - \PE[g(Z)]\eqsp,
    \end{equation}
    where \(Z\) is the \(d\)-dimensional standard Gaussian distribution. It can then be guaranteed that
    \begin{align}
    \|\Sigma^{-1/2}(\nabla^2 f_g(x) - \nabla^2 f_g(y))\Sigma^{-1/2}\|_e \lesssim\varepsilon^{-1}\,\lambda^{-1}_{\operatorname{min}}(\Sigma)\,\log^{3/2}(d)\,\|\Sigma^{1/2}(x-y)\|_\infty\eqsp.
    \end{align}
\end{proposition}
\begin{proof}
    Without loss of generality we can set \(\mu=0\). Starting from \cite{gallouet_mijoule_swan_stein}, we have that 
\begin{align}\label{eq:partial_repr}
        &\partial_{ij} f_g(x) - \partial_{ij} f_g(y) \\
        &\qquad=-\int_0^1 \frac{1}{2(1-t)}
\PE\!\left[
\bigl(Z_iZ_j - \delta_{ij}\bigr)
\bigl(
g(\sqrt{t}\,x + \sqrt{1-t}\,Z)
-
g(\sqrt{t}\,y + \sqrt{1-t}\,Z)
\bigr)
\right] dt \eqsp.
    \end{align}
For the \(\psi\in C^2\) Stein's identity gives \(\PE[(Z_iZ_j - \delta_{ij})\psi(Z)] = \PE[\partial_{ij}\psi(Z)]\). Then 
\begin{align}
    &\PE\!\left[
\bigl(Z_iZ_j - \delta_{ij}\bigr)
\bigl(
g(\sqrt{t}\,x + \sqrt{1-t}\,Z)
-
g(\sqrt{t}\,y + \sqrt{1-t}\,Z)
\bigr)
\right]\\
&\quad=(1-t)\PE[\bigl(
\partial_{ij}g(\sqrt{t}\,x + \sqrt{1-t}\,Z)
-
\partial_{ij}g(\sqrt{t}\,y + \sqrt{1-t}\,Z)
\bigr)]\eqsp.
\end{align}
Substituting into \eqref{eq:partial_repr} yields
\begin{align}
\partial_{ij} f_g(x) - \partial_{ij} f_g(y)
=-\frac{1}{2}
\int_0^1 
\PE\!\left[
\partial_{ij} g(\sqrt{t}\,x + \sqrt{1-t}\,Z)
-
\partial_{ij} g(\sqrt{t}\,y + \sqrt{1-t}\,Z)
\right] dt .
\end{align}
Fix $t\in(0,1)$. Since $g\in C^2(\mathbb R^d)$ and its second derivatives are integrable under the Gaussian shift, the differentiation-under-the-integral theorem applies and allows us to interchange $\partial_{ij}$ and $\PE$
\[
\PE\!\left[\partial_{ij} g(\sqrt t\,x+\sqrt{1-t}\,Z)\right]
=
\partial_{ij}\PE\!\left[g(\sqrt t\,x+\sqrt{1-t}\,Z)\right],
\]
\[
\PE\!\left[\partial_{ij} g(\sqrt t\,y+\sqrt{1-t}\,Z)\right]
=
\partial_{ij}\PE\!\left[g(\sqrt t\,y+\sqrt{1-t}\,Z)\right].
\]
Next we rewrite the expectation in terms of $\varphi_r(\,\cdot\,,\Sigma_t^{1/2})$ with
\[
\Sigma_t = (1-t)\Sigma + \varepsilon^2I\eqsp, \quad\varepsilon_t := \sqrt{(1-t)\lambda_{\operatorname{min}}(\Sigma)+\varepsilon^2}.
\]
Recall that $g(x)=\varphi_r(\mu + \Sigma^{1/2}x,\varepsilon)=\PE_\eta[\mathbf 1_{A_r}(\mu +\Sigma^{1/2}x+\varepsilon\eta)]$, where
$\eta\sim\mathcal N(0,I_d)$ is independent of $Z$.
Since the sum of independent Gaussian random vectors $\sqrt{1-t}\,\Sigma^{1/2}\,Z+\varepsilon\eta$
has the same distribution as $\Sigma_t^{1/2} Z$, we obtain
\begin{align}
\PE\!\left[g(\sqrt t\,x+\sqrt{1-t}\,Z)\right]
&=
\PE\!\left[\mathbf 1_{A_r}(\sqrt t\,\Sigma^{1/2}x+\sqrt{1-t}\,\Sigma^{1/2}Z+\varepsilon\eta)\right]
\\&=
\PE\!\left[\mathbf 1_{A_r}(\sqrt t\,\Sigma^{1/2}x+ \Sigma_t^{1/2} Z)\right]
=
\varphi_r(\sqrt t\,\Sigma^{1/2}x,\Sigma^{1/2}_t),
\end{align}
and analogously,
\[
\PE\!\left[g(\sqrt t\,y+\sqrt{1-t}\,Z)\right]
=
\varphi_r(\sqrt t\,\Sigma^{1/2}y,\Sigma^{1/2}_t).
\]
Combining this identity with the differentiation-under-the-integral step above yields
\begin{align}
&\PE\!\left[\nabla^2 g(\sqrt t\,x+\sqrt{1-t}\,Z)\right]
=
\Sigma^{1/2}\nabla^2\varphi_r(\sqrt t\,\Sigma^{1/2}x,\Sigma^{1/2}_t)\Sigma^{1/2}
\eqsp,\\
&\PE\!\left[\nabla^2 g(\sqrt t\,y+\sqrt{1-t}\,Z)\right]
=
\Sigma^{1/2}\nabla^2\varphi_r(\sqrt t\,\Sigma^{1/2}y,\Sigma^{1/2}_t)\Sigma^{1/2}\eqsp.
\end{align}
Summing over $i,j$ gives
\begin{align}
\|&\Sigma^{-1/2}(\nabla^2 f_g(x)-\nabla^2 f_g(y))\Sigma^{-1/2}\|_e\\
&
\lesssim \int_{0}^1\|\nabla^2\varphi_r(\sqrt t\,\Sigma^{1/2}x,\Sigma^{1/2}_t) - \nabla^2\varphi_r(\sqrt t\,\Sigma^{1/2}y,\Sigma^{1/2}_t)\|_e\,dt\\
&=
\int_0^1 
\sum_{i,j=1}^d
\Big|
\partial_{ij}\varphi_r(\sqrt t\,\Sigma^{1/2}x,\Sigma^{1/2}_t)
-
\partial_{ij}\varphi_r(\sqrt t\,\Sigma^{1/2}y,\Sigma^{1/2}_t)
\Big|\, dt \eqsp.
\end{align}
Applying the integral form of the mean value theorem to the map
$z\mapsto \partial_{ij}\varphi_r(z,\Sigma^{1/2}_t)$, we obtain
\begin{align}
\Big|&\partial_{ij}\varphi_r(\sqrt t\,\Sigma^{1/2}x,\Sigma^{1/2}_t)
-
\partial_{ij}\varphi_r(\sqrt t\,\Sigma^{1/2}y,\Sigma^{1/2}_t)
\Big|\\
&=
\Big|\int_0^1
\left\langle
\nabla \partial_{ij}\varphi_r\!\big(\sqrt t\,\Sigma^{1/2}(y+s(x-y)),\Sigma_t^{1/2}\big),
\ \sqrt t\,\Sigma^{1/2}(x-y)
\right\rangle ds \Big|\\
&\lesssim
\sqrt t\,\|\Sigma^{1/2}(x-y)\|_\infty
\int_0^1
\sum_{k=1}^d
\Big|
\partial_{kij}\varphi_r\!\big(\sqrt t\,\Sigma^{1/2}(y+s(x-y)),\Sigma_t^{1/2}\big)
\Big|\,ds \eqsp.
\end{align}
Collecting the above bounds, we obtain
\begin{align}
&\|\nabla^2 f_g(x)-\nabla^2 f_g(y)\|_e
\\
&\qquad\lesssim
\|\Sigma^{1/2}(x-y)\|_\infty
\int_0^1\!\!\int_0^1
\sqrt t
\sum_{k,i,j=1}^d
\Big|
\partial_{kij}\varphi_r\!\big(\sqrt t\,\Sigma^{1/2}(y+s(x-y)),\Sigma_t^{1/2}\big)
\Big|\,ds\,dt \eqsp.
\end{align}
Note that \(\Sigma_t\succeq\varepsilon_tI\).
By Lemma \ref{sum of derivatives lemma} with $s=3$ and Lemma \ref{lem:comparison_third_derivatives}, we have for any $t\in(0,1)$ and $s\in(0,1)$ that
\begin{align}
&\sum_{k,i,j=1}^d
\Big|
\partial_{kij}\varphi_r\!\big(\sqrt t\,\Sigma^{1/2}(y+s(x-y)),\Sigma_t^{1/2}\big)
\Big|\\
&\lesssim
\sup_{y\in\rset^{d}}\sum_{k,i,j=1}^d
\Big|
\partial_{kij}\varphi_r\!\big(y,\varepsilon_t\big)
\Big|
\lesssim
\varepsilon_t^{-3}\,\log^{3/2}(d)\eqsp.
\end{align}
Substituting this bound into the previous display yields
\begin{align}
&\|\Sigma^{-1/2}(\nabla^2 f_g(x)-\nabla^2 f_g(y))\Sigma^{1/2}\|_e
\lesssim
\|\Sigma^{1/2}(x-y)\|_\infty
\int_0^1\!\!\int_0^1
\sqrt t\,
\varepsilon_t^{-3}\,\log^{3/2}(d)\,ds\,dt \\
&\qquad=
\|\Sigma^{1/2}(x-y)\|_\infty\,\log^{3/2}(d)
\int_0^1
\sqrt t\,\varepsilon_t^{-3}\,dt \eqsp,
\end{align}
where we used that the integrand no longer depends on $s$.
Returning to the original smoothing parameter $\varepsilon$ and using the crude bound $\sqrt t\le 1$ we get
\begin{align}
\int_0^1 \sqrt t\,\varepsilon_t^{-3}\,dt
&=
\int_0^1
\frac{\sqrt t}{((1-t)\lambda_{\operatorname{min}}(\Sigma)+\varepsilon^2)^{3/2}}\,dt 
\le
\int_0^1
\frac{1}{((1-t)\lambda_{\operatorname{min}}(\Sigma)+\varepsilon^2)^{3/2}}\,dt \\
&=
\frac{2}{\lambda_{\operatorname{min}}(\Sigma)}\left(\frac{1}{\varepsilon}-\frac{1}{\sqrt{\lambda_{\operatorname{min}}(\Sigma)+\varepsilon^2}}\right)
\;\lesssim\;
\frac{\varepsilon^{-1}}{\lambda_{\operatorname{min}}(\Sigma)}\eqsp.
\end{align}
Collecting the above estimates, we conclude that
\[
\|\nabla^2 f_g(x)-\nabla^2 f_g(y)\|_e
\lesssim
\varepsilon^{-1}\,\lambda^{-1}_{\operatorname{min}}(\Sigma)\,\log^{3/2}(d)\,\|\Sigma^{1/2}(x-y)\|_\infty\eqsp,
\]
which proves the proposition.
\end{proof}

\begin{lemma}[Comparison for third derivatives]\label{lem:comparison_third_derivatives}
Let $\Gamma_1,\Gamma_2\in\mathbb S^{d\times d}_+$ satisfy $\Gamma_1\succeq \Gamma_2\succ 0$. Then for every $x\in\mathbb{R}^d$,
\begin{align}
\sum_{k,i,j=1}^d \big|\partial_{kij}\varphi_r(x,\Gamma_1^{1/2})\big|
\leq
\sup_{y\in\mathbb{R}^d}\sum_{k,i,j=1}^d \big|\partial_{kij}\varphi_r(y,\Gamma_2^{1/2})\big|.
\end{align}
\end{lemma}

\begin{proof}
Set $\Delta:=\Gamma_1-\Gamma_2\succeq 0$ and let \(\xi,\eta\sim\mathcal{N}(0, I_d)\) independent. Then
\begin{align}
\varphi_r(x,\Gamma_1^{1/2})
&=
\PE[\mathbf{1}_{A_r}(x+\Delta^{1/2}\xi+\Gamma_2^{1/2}\eta\in A_r)]\\
&
=\PE\big[\PE[\mathbf{1}_{A_r}(x+\Delta^{1/2}\xi+\Gamma_2^{1/2}\eta\in A_r)]\mid\xi\big]\\
&=\PE[\varphi_r(x+\Delta^{1/2}\xi, \Gamma^{1/2}_2)]\eqsp.
\end{align}
Differentiating three times with respect to $x$ under $\PE_\xi$ yields, for each $k,i,j$,
\[
\partial_{kij}\varphi_r(x,\Gamma_1^{1/2})
=
\PE\!\left[\partial_{kij}\varphi_r(x+\Delta^{1/2}\xi,\,\Gamma_2^{1/2})\right].
\]
Taking absolute values and summing over $k,i,j$ gives
\[
\sum_{k,i,j} \big|\partial_{kij}\varphi_r(x,\Gamma_1^{1/2})\big|
\leq\PE\Big[\sum_{k,i,j}\big|\partial_{kij}\varphi_r(x+\Delta^{1/2}\xi,\,\Gamma_2^{1/2})\big|\Big]\leq \sup_{y\in\mathbb{R}^d}\sum_{k,i,j=1}^d \big|\partial_{kij}\varphi_r(y,\Gamma_2^{1/2})\big|.
\]
\end{proof}

\paragraph{Central Limit Theorem with Non-Constant Predictable Variation}

Extending the result of \Cref{thm:clt_const_var} to the case of a non-constant predictable quadratic
variation requires concentration properties of the predictable quadratic
variation. To the best of our knowledge, such tail concentration cannot be
derived in full generality, and therefore it is imposed as part of the
assumptions in \Cref{lem:clt_general}. On the other hand, this concentration
property does hold in the Q-learning setting (\Cref{Q learn alg}), where the
martingale differences are functions of an underlying Markov chain, and hence
the predictable variation \(\PE_{k-1}[X_k X_k^{\top}]\) inherits the required
mixing structure.
\begin{theorem}\label{lem:clt_general}
    Let \Cref{assum:assumptions} hold. In addition, suppose that for all $k$,
\[\supnorm{X_k}\leq \kappa \quad\text{almost surely} \eqsp.\] Moreover, assume that there exist \(\varsigma > 0\) such that for all \(t > 0\)
\begin{align}
\label{eq:concentration_assumption}
\Pb \Big(\Big\|\sum_{k=1}^n\PE_{k-1}[X_k X_k^\top ] - \Sigma_n\Big\| \geq t \Big) \le 2d^2 \exp(-\varsigma nt^2)\eqsp. 
\end{align}
    Then it holds that
\begin{align}
     d_K(S_n, \mathcal{N}(0, \Sigma_n)) 
    &\leq \frac{[\ln_+ d]^{5/4}}{\underline{\sigma}^{1/2}(\Sigma_n)}
\Bigg(\frac{\overline{\sigma}(\Sigma_n)}{\sqrt{n}}+\sum_{k=1}^n\PE\frac{\|X_k\|_\infty^3 }{\lambda_{\min}({P}_k +\Sigma_n/n)}\Bigg)^{1/2}\\
&+  \frac{1}{(\varsigma n)^{1/2}}\Bigg( \frac{\log(d)\log^{1/2}(n)\|\Sigma_n\|}{\lambda_{\min}^2(\Sigma_n)} \Bigg) + \frac{1}{(\varsigma n)^{1/4}} \Bigg( \frac{\log(d)\log^{1/4}(n)\big\|\Sigma_n\big\|_2^{1/2}}{\underline{\sigma}(\Sigma_n)\lambda_{\min}^{1/2}(\Sigma_n)} 
\Bigg)\eqsp.
\end{align}
\end{theorem}

\begin{proof}
We adapt the arguments from \cite[p. 23]{belloni2018high}. We will requiere i.i.d. standard gaussian vectors \(Z_1, \ldots, Z_{n+1}\in \rset^d\) defined on the same probability space and independent from \(\{X_k\}_{k=1}^n\). Consider the following stopping time
\begin{align}
\label{lem:mart_lim:tau}
\tau = \max\{0 \leq k \leq n :\sum_{i=1}^k V_i \preceq  (1+t)\Sigma_n \} \eqsp ,
\end{align}
where $t \in \mathbb{R}_+$ is a parameter we will choose later. Introduce stopped martingale difference sequence  \(\{X_k'\}_{k=1}^n\), 
\begin{align}
    X_k' = X_{k}\cdot \mathbf{1}\{k\leq \tau\}\eqsp.
\end{align}
Define predictibale variance \(V_k' = \PE_{k-1}[X_k X_k'^\top]\). Then   quadratic variation process satisfies 
\[\sum_{k=1}^n V_k' =\sum_{k=1}^{\tau}V_k \preceq (1 + t)\Sigma_n\]
Our goal is to extend $\{X_k'\}_{k=1}^n$ to a sequence that has a constant quadratic characteristic equal to $(1+t)\Sigma_n$. To proceed, consider 
\begin{align}
\label{eq:mart_lim:spectral}
X_{n+1}' = \Big((1+t)\Sigma_n-\sum_{i=1}^\tau V_i \Big)^{1/2} Z_{n+1}\eqsp.
\end{align}
Thus, we obtain by the construction
\begin{align}
\sum_{k=1}^{n+1}V_k'= (1+t) \Sigma_n \eqsp. 
\end{align}
Set \(\Sigma'_n = (1+t)\Sigma_n\) and \(S_{n+1}' = \sum_{k=1}^{n+1} X_k'\eqsp\). Now we apply \Cref{lem:shao} and get
\begin{align}\label{eq:shao_mao}
    d_K(S_n, \mathcal{N}(0, \Sigma_n)) &\leq d_K(S_{n}, \mathcal{N}(0,\Sigma_n')) + d_K(\mathcal{N}(0, \Sigma_n'), \mathcal{N}(0, \Sigma_n))\\
    &\leq  2\PE^{1/(p+1)}[\supnorm{S_n-S_{n+1}'}^p]\left( \frac{2(\sqrt{2 \log d} + 2)]}{\underline{\sigma}(\Sigma_n)} \right)^{p/(p+1)}\\
    &+ d_K(S_{n+1}', \mathcal{N}(0,\Sigma_n')) + d_K(\mathcal{N}(0, \Sigma_n'), \mathcal{N}(0, \Sigma_n)) \eqsp.
\end{align}
\Cref{lem:gaussian_comparison} implies 
\begin{align}\label{lem_21_1}
    d_K(\mathcal{N}(0, \Sigma_n'), \mathcal{N}(0, \Sigma_n)) \lesssim \frac{t\|\Sigma_n\|}{\lambda_{\min}(\Sigma_n)} \log d
\Big( 1 \vee \Big|\log\!\Big( \frac{t\|\Sigma_n\|}{\lambda_{\min}(\Sigma_n)} \Big)\Big| \Big)\eqsp.
\end{align}
Applying Minkowski's inequality we obtain
\[\PE^{1/p} [\supnorm{S_n - S_{n+1}'}^p\leq \PE^{1/p}\supnorm{S_n - S_{n}'}^p + \PE^{1/p}[\supnorm{X_{n+1}'}^p]\]
To control the moments of $\PE^{1/p} [\supnorm{S_n - S_n'}^p$ we consider two events:
\begin{align}
    \Omega_1 = \{\omega: \|{\sum_{i=1}^n V_i - \Sigma_n} \|_2 \leq t\,\lambda_{\min}(\Sigma_n) \} \eqsp, \quad \Omega_2 = \{\omega: \|\sum_{i=1}^n V_i - \Sigma_n\|_2 > t\,\lambda_{\min}(\Sigma_n) \} \eqsp.
\end{align}
On the event $\Omega_1$ it holds that $\sum_{i=1}^nV_i \preceq (1+t)\Sigma_n$. Thus, $\tau = n$ and \(S_n =S_n'\).
Hence,
\begin{align}
    \label{lem:convex_clt_Smoments}
    \PE^{1/p}[\supnorm{S_n - S_n'}^{p}] =  \PE^{1/p}[\supnorm{S_n - S_n'}^{p}  \mathbf{1}\{\Omega_2\}]\leq  n\kappa \,\Pb(\Omega_2)^{1/p}\eqsp. 
\end{align}
By assumption \eqref{eq:concentration_assumption} we obtain
\begin{align}\label{eq:Omega_2_prob}
\Pb(\Omega_2)
=\Pb\Big(\big\|\sum_{k=1}^n V_k-\Sigma_n\big\|_2> t\,\lambda_{\min}(\Sigma_n)\Big)
\le 2d^2\exp(-\varsigma nt^2\lambda_{\min}^2(\Sigma_n)).
\end{align}
Hence 
\begin{align}
     \PE^{1/p}[\supnorm{S_n - S_n'}^{p}] \leq  n\kappa  \,d^{2/p}\exp\Big(\frac{-\varsigma n t^2\lambda_{\min}^2(\Sigma_n)}{p}\Big)\eqsp. 
\end{align}
\Cref{lem:gaussian_max_bound} allow  us to control the moments of \(X_{n+1}'\)
\begin{align}
    \PE^{1/p}[\supnorm{X_{n+1}'}^p] = \PE^{1/p}[\PE[\supnorm{X_{n+1}'}^p\mid \mathcal{F}_{n}]] \lesssim
    \sqrt{p\log d}\, \PE^{1/p}\Big[\overline{\sigma}^p\Big\{ (1+t)\Sigma_n-\sum_{i=1}^\tau V_i \Big\}\Big]
\end{align}
To estimate maxima of covariance matrix we again split probabolity space on \(\Omega_1\) and \(\Omega_2\) and use estimation \eqref{eq:Omega_2_prob}:
\begin{align}\label{lem_21_3}
&\PE^{1/p}\Big[\overline{\sigma}^p\Big\{ (1+t)\Sigma_n-\sum_{i=1}^\tau V_i \Big\}\Big]\\
&\qquad\leq \PE^{1/p}\Big[\overline{\sigma}^p\Big\{ (1+t)\Sigma_n-\sum_{i=1}^\tau V_i \Big\}\mathbf{1}\{\Omega_1\}\Big] + \PE^{1/p}\Big[\overline{\sigma}^p\Big\{ (1+t)\Sigma_n-\sum_{i=1}^\tau V_i \Big\}\mathbf{1}\{\Omega_2\}\Big] \\
&\qquad\lesssim \PE^{1/p}\Big[\Big\|\Sigma_n-\sum_{k=1}^n V_k \Big\|^{p/2}_2\Big] + \sqrt{t}\big\|\Sigma_n\big\|_2^{1/2} + \overline{\sigma}\Big\{(1+t)\Sigma_n\Big\}d^{2/p}\exp\Big(-\frac{\varsigma n t^2\lambda_{\min}^2(\Sigma_n)}{p}\Big)\\
&\qquad \lesssim \frac{p^{1/2}d^{1/p}}{(\varsigma n)^{1/4}}+ \sqrt{t}\big\|\Sigma_n\big\|_2^{1/2} + \overline{\sigma}\Big\{(1+t)\Sigma_n\Big\} d^{1/p}\exp\Big(-\frac{\varsigma nt^2\lambda_{\min}^2(\Sigma_n)}{p}\Big)
\end{align}
Note that $\{X'_{k}\}_{k=1}^{n+1}$ satisfies the assumptions of \Cref{thm:clt_const_var}. Moreover, since by construction \(X_{n+1}'\) gaussian random vector, the last term in the Lindeberg decomposition \eqref{eq:telescope} cancels. Hence, setting \(\Sigma = \Sigma_n/n\) we obtain
\begin{align}\label{eq:clt_general_constant}
    d_K(S_{n+1}'\,,\mathcal{N}(0, \Sigma_n'))\lesssim \frac{[\ln_+ d]^{5/4}}{\underline{\sigma}^{1/2}(\Sigma_n')}
\Bigg(\frac{\overline{\sigma}(\Sigma_n)}{\sqrt{n}}+\sum_{k=1}^n\PE\frac{\|X_k'\|_\infty^3 }{\lambda_{\min}({P}_k' +\Sigma_n / n)}\Bigg)^{1/2}\eqsp.
\end{align}
Observe that on the event \(\Omega_1\) we have \(\tau=n\), and hence
\(X_k = X_k'\) for \(k\in\{1,\ldots,n\}\).
Consequently, \(P_k' \succeq P_k\). On the event \(\Omega_2\) we use crude bound \(\lambda_{\min}(P_k' + \Sigma) \geq \lambda_{\min}(\Sigma)   \) and probability estimation in \eqref{eq:Omega_2_prob}. Thus
\begin{align}\label{eq:clt_part}
    d_K(S_{n+1}'\,,\mathcal{N}(0, \Sigma_n'))&\lesssim \frac{[\ln_+ d]^{5/4}}{\underline{\sigma}^{1/2}(\Sigma_n)}
\Bigg(\frac{\overline{\sigma}(\Sigma_n)}{\sqrt{n}}+\sum_{k=1}^n\PE\frac{\|X_k\|_\infty^3 }{\lambda_{\min}({P}_k+\Sigma_n/n)}\Bigg)^{1/2} \\
&+ \frac{[\ln_+ d]^{5/4}}{\underline{\sigma}^{1/2}(\Sigma_n)}
\Bigg(\frac{\overline{\sigma}(\Sigma_n)}{\sqrt{n}}+\frac{n^2\kappa^3 \Pb(\Omega_2)}{\lambda_{\min}(\Sigma_n) }\Bigg)^{1/2} 
\end{align}

Now choose
\[p := \log(d)\qquad t^2 = \frac{\log(d)\log(n)}{\varsigma n\lambda_{\min}^2(\Sigma_n)} \] and subsitute \eqref{lem_21_1},  \eqref{lem_21_3} and \eqref{eq:clt_part} into \eqref{eq:shao_mao}
\begin{align}
     d_K(S_n, \mathcal{N}(0, \Sigma_n)) 
    &\leq \frac{[\ln_+ d]^{5/4}}{\underline{\sigma}^{1/2}(\Sigma_n)}
\Bigg(\frac{\overline{\sigma}(\Sigma_n)}{\sqrt{n}}+\sum_{k=1}^n\PE\frac{\|X_k\|_\infty^3 }{\lambda_{\min}({P}_k +\Sigma_n/n)}\Bigg)^{1/2}\\
&+  \frac{1}{(\varsigma n)^{1/2}}\Bigg( \frac{\log(d)\log^{1/2}(n)\|\Sigma_n\|}{\lambda_{\min}^2(\Sigma_n)} \Bigg) + \frac{1}{(\varsigma n)^{1/4}} \Bigg( \frac{\log(d)\log^{1/4}(n)\big\|\Sigma_n\big\|_2^{1/2}}{\underline{\sigma}(\Sigma_n)\lambda_{\min}^{1/2}(\Sigma_n)} 
\Bigg)\eqsp.
\end{align}

\end{proof}

\subsection{Gaussian Comparison, maximum norm and anti-concentration}

\begin{lemma}[Lemma 2.3 in \cite{KOJEVNIKOV2022109448}]\label{lem:gaussian_max_bound}
Let \(Y = [Y_1,\dots,Y_d]^\top\) be a zero-mean Gaussian vector in
\(\mathbb{R}^d\), with
\(
\sigma_j^2 := \mathbb{E} Y_j^2 > 0\)
for all \( 1 \le j \le d,
\)
and let
\(
\overline{\sigma} := \max_{1 \le j \le d} \sigma_j.
\)
Then for any \(p \ge 2\),
\begin{equation}\label{eq:gaussian_max_moment}
\PE^{1/p}\|Y\|_\infty^p
\;\le\;
\overline{\sigma}\sqrt{p\log d}\eqsp.
\end{equation}
\end{lemma}

\begin{lemma}[Theorem 1.1 in \cite{10.1214/20-AAP1629}]\label{lem:gaussian_comparison}
Let $Z_1 \sim \mathcal{N}(0,\Sigma_1)$ and let
$Z_2 \sim \mathcal{N}(0,\Sigma_2)$. Denote \(\|\Sigma_1 - \Sigma_2\|_{\operatorname{ch}} = \max_{ij}|\Sigma_1(i,j) - \Sigma_2(i,j)|\). Then
\[
\sup_{A \in \mathcal{R}}
\bigl| \mathbb{P}(Z_1 \in A) - \mathbb{P}(Z_2 \in A) \bigr|
\;\le\;
C \frac{\|\Sigma_1 - \Sigma_2\|_{\operatorname{ch}}}{\lambda_{\operatorname{min}}(\Sigma_1)} \log d
\Big( 1 \vee \Big|\log\!\Big( \frac{\|\Sigma_1 - \Sigma_2\|_{\operatorname{ch}}}{\lambda_{\operatorname{min}}(\Sigma_1)} \Big)\Big| \Big)\eqsp.
\]

\end{lemma}

\begin{lemma}
\label{lem:shao}
Let \(Y\) be a centered Gaussian vector in \(\mathbb{R}^d\) such that $\min_{1 \le j \le d} \PE[Y_j^2] \geq \underline \sigma^2$ for some $\sigma > 0$. Then for any random vector \(X, X^\prime\) taking values in \(\mathbb{R}^d\), and any \(p \geq 1\),
\begin{align*}
\sup_{h = 1_A, A \in \mathcal R} \left|\PE[h(X + X^\prime)] - \PE[h(Y)]\right| 
&\leq \sup_{h = 1_A, A \in \mathcal R} \left|\PE[h(X)] - \PE[h(Y)]\right| \\
&+  2\PE^{1/(p+1)}[\supnorm{X'}^p]\left( \frac{2(\sqrt{2 \log d} + 2)]}{\underline{\sigma}} \right)^{p/(p+1)}\, .
\end{align*}
\end{lemma}
\begin{proof}
First, note that by \cite{chernozhukov2017detailedproofnazarovsinequality}[Lemma 2.1]
\begin{equation}
\label{anti concentration}
 \Pb(Y \le y + \varepsilon) - \Pb(Y \le y) \le \frac{\varepsilon}{\underline{\sigma}} (\sqrt{ 2\log d} + 2) \, .   
\end{equation}
Note that for any $A \in \mathcal R$,
$$
\Pb(X + X^\prime \in A) \le \Pb(X \in A_\varepsilon) + \frac{\PE[\|X^\prime\|_\infty^p]}{\varepsilon^p} \, ,
$$
where $A_\varepsilon = \prod_{j=1}^d (-\infty, a_j + \varepsilon]$. 
Hence, by \eqref{anti concentration}, 
$$
\Pb(X + X^\prime \in A)  - \Pb(Y \in A) \le \sup_{h = 1_A, A \in \mathcal R} \left|\PE[h(X)] - \PE[h(Y)]\right| + \frac{2\varepsilon}{\underline{\sigma}} (\sqrt{ 2\log d} + 2) + \frac{\PE[\|X^\prime\|_\infty^p]}{\varepsilon^p}\, .
$$
Choosing
$$
\varepsilon = \left( \frac{\underline{\sigma} \PE[\|X^\prime\|_\infty^p]}{2(\sqrt{2 \log d} + 2)} \right)^{1/(p+1)}\eqsp,
$$
we obtain
\begin{align*}
\sup_{h = 1_A, A \in \mathcal R} (\Pb(X + X^\prime \in A)  - \Pb(Y \in A)) & \le \sup_{h = 1_A, A \in \mathcal R} \left|\PE[h(X)] - \PE[h(Y)]\right| \\
&+ 2\PE^{1/(p+1)}[\supnorm{X'}^p]\left( \frac{2(\sqrt{2 \log d} + 2)]}{\underline{\sigma}} \right)^{p/(p+1)} \, .
\end{align*}
Similarly, we may estimate $\sup_{h = 1_A, A \in \mathcal R}(\PE[h(Y)] - \PE[h(X + X^\prime)])$.
\end{proof}

%% file: appendix_auxillary.tex
\subsection{Properties of step-size sequence}
\label{appendix:step_size}

For the proofs of \Cref{lem:bsum}--\Cref{lem:P_alpha_ineq}, we refer the reader to
Appendix~I of \cite{samsonov2025statistical}.
\begin{lemma}\label{lem:bsum}
Let \(b > 0\) and let \(\{\alpha_k\}_{k \ge 0}\) be a non-increasing sequence such that 
\(\alpha_0 \le 1/b\). Then it holds that
\[
\sum_{j=1}^{k} \alpha_j \prod_{l=j+1}^{k} (1 - \alpha_l b)
=
\frac{1}{b} \left\{ 1 - \prod_{l=1}^{k} (1 - \alpha_l b) \right\}.
\]
\end{lemma}
\begin{lemma}
\label{lem:sum_as_Qell}
Let $b, c_0 > 0$ and \(\alpha_k = \frac{c_0}{(k + k_0)^{\omega}}\) with \(\omega \in (\frac{1}{2},1)\), $k_0 \geq 0$. Assume that $bc_0 \leq\frac{1}{2}$ and $k_0^{1-\omega} \geq 2/(bc_0)$. Then it holds that
\begin{align}
\sum_{k=\ell}^{n-1} \alpha_\ell\prod_{j=\ell+1}^{k} (1-b\alpha_j) \leq c_0 + \frac{2}{b(1-\omega)}\eqsp.
\end{align}

\end{lemma}
\begin{lemma}\label[lemma]{lem:rate_of_convergence}
Let $b, c_0 > 0$ and \(\alpha_k = \frac{c_0}{(k + k_0)^{\omega}}\) with \(\omega \in (\frac{1}{2},1)\), $k_0 \geq 0$. Assume that $bc_0 \leq\frac{1}{2}$ and $k_0^{1-\omega} \geq 8/(bc_0)$. 
Then for any $q \in (1;3]$, it holds that 
\[
\sum_{j=1}^{k} \alpha_j^{q} 
\prod_{\ell=j+1}^{k} (1 - \alpha_\ell b)
\leq
\frac{4}{b}\,\alpha_k^{q-1}.
\]
\end{lemma}

\begin{lemma}\label{lem:alpha_delta}
Let $b, c_0 > 0$ and \(\alpha_k = \frac{c_0}{(k + k_0)^{\omega}}\) with \(\omega \in (\frac{1}{2},1)\), $k_0 \geq 0$. Assume that $bc_0 \leq\frac{1}{2}$ and $k_0^{1-\omega}\geq 8/(bc_0)$. Then it holds that
\[
    \frac{\alpha_k}{\alpha_{k+1}} \le 1 + b \alpha_{k+1}.
\]
\end{lemma}

\begin{lemma}\label{lem:P_alpha_ineq}
Let $b, c_0 > 0$ and \(\alpha_k = \frac{c_0}{(k + k_0)^{\omega}}\) with \(\omega \in (\frac{1}{2},1)\), $k_0 \geq 0$. Assume that $bc_0 \leq\frac{1}{2}$ and $k_0^{1-\omega}\geq 8/(bc_0)$. Then it holds that
\[
    \alpha_j \prod_{l=j+1}^{k} (1 - \alpha_l b) \le \alpha_k \eqsp.
\]
\end{lemma}

\begin{lemma}
\label{lem:telescope}
Let $(v_j)_{j \geq -1}$ be a sequence of $D$-dimensional vectors  Then it holds that
\begin{align}
\sum_{j=0}^{t} \alpha_j \Gamma_{j+1:t} (v_j -v_{j-1}) = \alpha_tv_t - \alpha_0P_{1:t}v_{-1} + \sum_{j=0}^{t-1}\big((\alpha_j - \alpha_{j+1}) - \alpha_j\alpha_{j+1}D_{\mu}(I - \gamma\MKQ^{\pi^\star})\big)\Gamma_{j+2:t}v_j\eqsp.
\end{align}
\end{lemma}
\begin{proof}
We commence the proof by reindexing the summation:
    \begin{align}
        \sum_{j=0}^{t} \alpha_j \Gamma_{j+1:t} (v_j -v_{j-1}) &= \alpha_tv_t - \alpha_0v_{-1} + \sum_{j=0}^{t-1}(\alpha_j\Gamma_{j+1:t} - \alpha_{j+1}\Gamma_{j+2:t})v_j\eqsp.
    \end{align}
We now analyze the coefficient term:
    \begin{align}
        \alpha_j\Gamma_{j+1:t} - \alpha_{j+1}\Gamma_{j+2:t}& = \alpha_j(I -\alpha_{j+1}D_{\mu}(I - \gamma\MKQ^{\pi^\star}))\Gamma_{j+2:t} - \alpha_{j+1}\Gamma_{j+2:t}\\
        &= \big((\alpha_j - \alpha_{j+1}) - \alpha_j\alpha_{j+1}D_{\mu}(I - \gamma\MKQ^{\pi^\star})\big)\Gamma_{j+2:t}\eqsp.
    \end{align}
Substituting this expression back into the summation yields the desired result.
\end{proof}

%% file: appendix_concentration.tex
\subsection{Concentration inequality for martingales}
To establish bounds on arbitrary moments of \(\xi_t^{(0)}\), we now derive a modified version of the Burkholder-Davis-Gundy inequality tailored for the supremum norm. In contrast to the classical formulation, the consideration of the supremum norm introduces an additional logarithmic dependence on the dimension of the martingale sequence. The proof of \Cref{lem:bdg_1} closely follows the argument of \cite[Lemma~H.1]{li2023statistical}, where an analogous bound was established for the case $p=1$. 
\begin{lemma}
\label{lem:bdg_1}
Assume $\{X_j\} \subset \mathbb{R}^d$ are martingale differences adapted to the filtration 
$\{\mathcal{F}_j\}_{j \ge 0}$ with zero conditional mean $\mathbb{E}[X_j | \mathcal{F}_{j-1}] = 0$ 
and finite conditional variance $V_j = \mathbb{E}[X_j X_j^{\top} | \mathcal{F}_{j-1}]$. 
Moreover, assume $\{X_j\}_{j \ge 0}$ is uniformly bounded, i.e. there exists $\varkappa > 0$ such that $\sup_j \|X_j\|_\infty \leq \varkappa$. For any sequence of deterministic matrices $\{B_j\}_{j \ge 0} \subset \mathbb{R}^{d \times d}$ 
satisfying $\sup_j \|B_j\|_\infty \le B$, define the weighted sum as
\[
Y_T = \sum_{j=1}^T B_j X_j
\]
and let 
\(
W_T = \operatorname{diag}\,(\sum_{j=1}^T B_j V_j (B_j)^{\top})
\)
be a diagonal matrix that collects conditional quadratic variations. Then, it follows that
\begin{equation}
\mathbb{E}^{1/p}[\|Y_T\|_{\infty}^p] \leq 4{p}\log(2dT^2)\,\big(\mathbb{E}^{1/p}\big[\|\mathbf{W}_T\|_{\infty}^{p/2}\big] + B \varkappa \big) \eqsp.
\end{equation}
\end{lemma}

\begin{proof}
Define event 
\[
\mathcal{H}_K = 
\left\{
\|\mathbf{Y}_T\|_\infty 
\ge 
\frac{2B \varkappa}{3}\log\frac{2dK}{\delta}
+ 
\sqrt{
4\max\!\left\{
\|\mathbf{W}_T\|_\infty,\,
\frac{T B^2 \varkappa^2}{2^K}
\right\}
\log\frac{2dK}{\delta}
}
\right\}.
\]
As established in \cite{li2023statistical}, the probability of this event satisfies \(\mathbb{P}(\mathcal{H}_K) \leq \delta\). Now, set $\delta=\tfrac{1}{T^p}$ and $K=\lceil\log_2 T\rceil $. Then, the following holds:
\begin{align}
&\mathbb{E}^{1/p}[\|\mathbf{Y}_T\|_{\infty}^p] 
\overset{(a)}{\leq} \mathbb{E}^{1/p}[\|\mathbf{Y}_T\|_{\infty}^p \mathbf{1}\{{\mathcal{H}_K}\} 
  + \mathbb{E}^{1/p}[\|\mathbf{Y}_T\|_{\infty}^p \mathbf{1}\{\mathcal{H}_K^c\}]  \\
&\quad\overset{(b)}{\leq}
T B \varkappa \,\mathbb{P}(\mathcal{H}_K)^{1/p}
+ \mathbb{E}^{1/p}\!\bigg[\bigg(
\frac{2 B \varkappa}{3}\log\frac{2dK}{\delta}
+ \sqrt{ 4\max\!\left\{
\|\mathbf{W}_T\|_{\infty},\,
\frac{TB^2\varkappa^2}{2^K}
\right\}
\log\frac{2dK}{\delta}
}
\bigg)^p\bigg] \\
&\quad \overset{(c)}{\leq}
B \varkappa + \frac{2pB\varkappa}{3}\log(2dT^2)
+ 2\sqrt{p}\,\mathbb{E}^{1/p}\left[\left(
\max\!\left\{
\|\mathbf{W}_T\|_{\infty},\,
B^2 \varkappa^2
\right\}
\log(2dT^2)
\right)^{p/2}\right]  \\
&\quad\overset{(d)}{\leq}
4 p B \varkappa \log(2dT^2)
+ 2\sqrt{p}\log(2dT^2)\,\mathbb{E}^{1/p}\big[\|\mathbf{W}_T\|_{\infty}^{p/2}\big]\eqsp.
\end{align} 
Inequality \((a)\) follows from the Minkowski inequality; \((b)\) utilizes the bound \(\|\mathbf{Y}_T\|_{\infty} \leq T B \varkappa\) and the definition of the event \(\mathcal{H}_K\); \((c)\) is justified by the bound \(\mathbb{P}(\mathcal{H}_K) \leq \frac{1}{T^p}\), the specific choices \(\delta = \frac{1}{T^p}\) and \(K = \lceil\log_2 T\rceil\) (which implies \(\frac{T}{2^K} \leq 1\) and \(K \leq \log_2 T + 1 \leq T^2\) for \(T \geq 2\)), and the application of the elementary inequality \((a+b)^2 \leq 2(a^2 + b^2)\) combined with Minkowski inequality for \(p/2\); finally, \((d)\) follows from further simplification and the absorption of lower-order terms.

\end{proof}
However, in our analysis, it is more consistent to express bounds in terms of the moments  \(\E^{1/p}[\supnorm{X_j}^p]\) rather than in terms of the conditional covariance $\E^{1/p}[\supnorm{W_T}^{p/2}]$. \Cref{lem:bdg_2} provides a relation between these quantities.

\begin{lemma}
\label[lemma]{lem:bdg_2}
Assume $\{X_j\} \subset \mathbb{R}^d$ are martingale differences adapted to the filtration 
$\{\mathcal{F}_j\}_{j \ge 0}$ with zero conditional mean $\mathbb{E}[X_j | \mathcal{F}_{j-1}] = 0$ 
and finite conditional variance $V_j = \mathbb{E}[X_j X_j^{\top} | \mathcal{F}_{j-1}]$. 
Moreover, assume that $\{X_j\}_{j \ge 0}$ are uniformly bounded, i.e. there exists $\varkappa > 0$ such that $\sup_j \|X_j\|_\infty \leq \varkappa$. Then the following holds:
\begin{equation}
\mathbb{E}^{1/p}[\|\sum_{j=1}^{T}X_j\|_{\infty}^p] \leq 4p\log(2dT^2)\,\Big(\varkappa + \Big(\sum_{j=1}^t \E^{2/p}[\|X_j\|_{\infty}^p]\Big)^{1/2}\Big) \eqsp.
\end{equation}
\end{lemma}
\begin{proof}
We first apply \Cref{lem:bdg_1} with \(B_i := I\)
\begin{equation}\label{eq:bdg_step}
    \mathbb{E}^{1/p}[\supnorm{Y_T}^p] \leq 4p\log(2dT^2)\,\big(\mathbb{E}^{1/p}\big[\supnorm{\mathbf{W}_T}^{p/2}\big] + \varkappa\big) \eqsp,
\end{equation}
where \( \mathbf{W}_T \) defined in \Cref{lem:bdg_1}.
We proceed to simplify the right-hand side through the following chain of inequalities
\begin{align}\label{eq:w_t_bound}
    \supnorm{W_t} &= \supnorm{\operatorname{diag}(\sum_{j=1}^t \E [X_j X_j^\top|\mathcal{F}_{j-1}])} 
    = \supnorm{\sum_{j=1}^t \operatorname{diag}(\E[X_j X_j^\top |\mathcal{F}_{j-1}])} \\
    &\leq \sum_{j=1}^t \supnorm{\operatorname{diag}(\E[X_j X_j^\top |\mathcal{F}_{j-1}])} = \sum_{j=1}^{t}\E[\supnorm{X_j}^2 \mid\mathcal{F}_{j-1}]\eqsp.
\end{align}
Next establish bounds for the moments with \(p\geq 2\) 
\begin{align}\label{eq:w_t_2}
    \E^{2/p}[\supnorm{W_t}^{p/2}] &\overset{(a)}{\leq} \E^{2/p} \bigg[\big(\sum_{j=1}^{t}\E\big[\supnorm{X_j}^2 \mid\mathcal{F}_{j-1}]\big]\big)^{p/2}\bigg] \overset{(b)}{\leq} \sum_{j=1}^t \E^{2/p}\bigg[\E^{p/2}\big[\supnorm{X_j}^2 \mid\mathcal{F}_{j-1}]\big]\bigg] \\
    &\overset{(c)}{\leq} \sum_{j=1}^{t}\E^{2/p} [\E\big[\supnorm{X_j}^p \mid\mathcal{F}_{j-1}]] \overset{(d)}{=} \sum_{j=1}^t \E^{2/p}[\supnorm{X_j}^p] \eqsp,
\end{align}
where in \((a)\) we use \eqref{eq:w_t_bound}, \((b)\) follows from the Minkowski inequality, \((c)\) from Jensen's inequality for conditional expectations, and \((d)\) utilizes the law of total expectation. The resulting bound is 
\begin{equation}\label{eq:w_t_final}
    \E^{1/p}[\|W_t\|_{\infty}^{p/2}] \leq \Big(\sum_{j=1}^t \E^{2/p}[\|X_j\|_{\infty}^p]\Big)^{1/2}\eqsp.
\end{equation}
Combining \eqref{eq:w_t_final} with \eqref{eq:bdg_step} yields the statement of the lemma.
\end{proof}